\documentclass[11pt, twoside]{article}
\usepackage[utf8]{inputenc}
\usepackage{amsmath}
\usepackage{amssymb}
\usepackage{amsthm}
\usepackage{enumitem, multirow, booktabs}
\usepackage[margin = 1in]{geometry}
\usepackage{secdot} %Toggle for a dot after section number
\usepackage{authblk}
\usepackage{changepage, booktabs}
\usepackage{hyperref}
\usepackage{url, diagbox}
\usepackage{mathrsfs}
\usepackage{amsthm}
\usepackage[normalem]{ulem}
\usepackage{enumitem}
\usepackage{comment}
\usepackage{natbib}
\usepackage{babel}%For a dot after section number in table of contents
\usepackage{graphicx}
\usepackage{bm} %bold math
\usepackage{cancel} %to strike out
\usepackage{caption}
\usepackage[table,xcdraw]{xcolor}
\usepackage{titletoc}
\usepackage{mathtools}
\usepackage{graphicx}
\usepackage{parskip}
\usepackage{algpseudocode}
\usepackage[ruled,vlined]{algorithm2e}
\usepackage{setspace}
\DeclareCaptionFormat{custom}
{%
    \textbf{#1#2}\textit{\small #3}
}
\graphicspath{ {./images/} }

\newcommand{\E}{\mathbb{E}}

\newcommand{\N}{\mathbb{N}}
\renewcommand{\P}{\mathbb{P}}

\newcommand{\R}{\mathbb{R}}
\renewcommand{\phi}{\varphi}
\theoremstyle{definition}

\newtheorem{remark}{Remark}[section]
\newtheorem{theorem}{Theorem}[section]
\newtheorem{proposition}{Proposition}[section]

\renewcommand{\hat}{\widehat}
\renewcommand{\tilde}{\widetilde}

\DeclareMathOperator*{\argmin}{arg\,min}

\allowdisplaybreaks

\makeatother
\usepackage{thmtools}
\usepackage[framemethod=TikZ]{mdframed}
\mdfsetup{skipabove=1em,skipbelow=0em}

\usepackage{cleveref}
\theoremstyle{definition}

\begin{document}

\begin{center}
    {\LARGE\bfseries Data Driven Block Replacement Scheduling}\\[1.2em]
    
    {\large
        Aniruddhan Ganesaraman \qquad \ Vidyadhar Kulkarni
    }\\[0.5em]

    {\normalsize \href{mailto: aniruddhan_ganesaraman@unc.edu}{\textit{aniruddhan\_ganesaraman@unc.edu}} \quad  \href{mailto: vkulkarn@email.unc.edu}{\textit{vkulkarn@email.unc.edu}}}\\[1em]
    
    {\normalsize
        Department of Statistics and Operations Research\\
        University of North Carolina at Chapel Hill
    }\\[0.3em]
    
    {\normalsize\textit{Submitted to Operations Research}}\\[0.3em]
    
    {\normalsize June 27, 2026}
\end{center}

%\maketitle

\subsection*{Abstract.} We develop data-driven algorithms for maintaining $N$ independent identical machines under a \textit{block replacement policy}, in which each machine is replaced upon failure and all machines are jointly replaced at regular intervals of length $k$. The goal is to learn the cost-minimizing interval $k^*$ from operational data when the lifetime distribution is unknown. At each decision epoch, the operator selects $k \in \{1, 2, \ldots, K\}$, observes the resulting failure history (a mixture of complete and right-censored lifetimes) and incurs a per-unit-time cost governed by the renewal function. We formulate this as a stochastic multi-armed bandit and propose Hoeffding- and Bernstein-based lower-confidence-bound algorithms achieving $O(K \log T)$ regret, matching the Lai--Robbins lower bound. Exploiting a nested observation property unique to block replacement, correlated variants attain $O((K-k^*)\log T)$ regret and require only $O(1)$ direct pulls of suboptimal arms $k < k^*$. A complementary Kaplan--Meier renewal algorithm estimates the lifetime distribution nonparametrically from censored data, achieving almost-sure policy consistency and empirically near-zero incremental regret at long horizons. We additionally analyze two average-cost MDPs: a time-elapsed formulation establishing that block replacement is optimal within its policy class for any lifetime distribution, and an age-vector formulation proving a monotone threshold structure under increasing failure rate distributions and providing a gold-standard cost benchmark. Numerical experiments confirm the theoretical ordering and reveal structural cost gaps between optimal block and age-dependent replacement.

Code for replication is publicly available at \href{https://github.com/AniruddhanG/Block-Replacement-Scheduling}{\textit{this Github repository}}. 

\textbf{Keywords.} Block replacement; Multi-armed bandits; Censored lifetime data; Data-driven policy.

\section{Introduction}
Many organizations face the problem of managing a large number of identical physical assets, all of which must remain functional at all times. For example, all the hard drives in a data center must remain operational to ensure continuous service availability. In aviation, every turbine blade in a commercial aircraft engine must remain in working condition for safe flight operations. Along a stretch of railroad track, every fastening clip must remain secure to maintain safe train passage, and in ground vehicles, every brake pad must remain functional to meet safety standards. One simple way to achieve this is to replace each item when it fails. This involves dispatching a repair crew to the site and replacing the failed item(s). This is called an unplanned replacement. It can be expensive due to the fixed cost of mobilizing a crew for a single failure. One way to mitigate this is to replace some of the unfailed items when a crew visits to repair a failed one. However, this involves maintaining a list of which items to replace on each visit, and keeping such a list may be impossible or too costly and cumbersome. A simpler strategy is to replace all the items at the site at regular intervals, regardless of whether they have failed. This is called a block replacement policy. Thus each turbine blade in an aircraft engine is replaced when it fails, but every so often, in an engine overhaul (say every three years), all blades are replaced, whether failed or not. Likewise, each railroad fastener is replaced when found defective during inspection, but every so often (say every five years) all fasteners along an entire track section are replaced wholesale. The main advantage of a block replacement policy is that it is easy to implement, since it is described by a single parameter, the inter-replacement interval $k$. A block replacement policy with parameter $k$ (denoted $\mathrm{BR}(k)$) replaces each item upon failure and schedules a full replacement of all items every $k$ time units. 

Clearly, the question is how to determine an optimal $k$, denoted by $k^*$. To determine this, we need a cost model that quantifies the costs involved. The  simplest model is to assume that each failure (which leads to an unplanned replacement) costs $c_f$ dollars, while each block replacement costs $c_b$ dollars. Note that a smaller $k$ reduces the cost of unplanned replacements, but increases the cost of block replacement. A larger $k$ has the opposite effect. This implies that there is an optimal inter-replacement interval $k^*$ that minimizes the long-run cost per unit time (cost rate).

It is clear that the cost rate of the policy $\mathrm{BR}(k)$ depends on the lifetimes of the items. We consider the simplest model: the lifetimes of all items are independent and identically distributed (iid) random variables with a common cumulative distribution function (cdf) $F$. If $F$ is known, it is easy to compute $c(k)$, the long-run cost rate of $\mathrm{BR}(k)$, using renewal reward processes. See Section~\ref{sec:setup} for details. The problem of finding an optimal $k$ then reduces to a one-dimensional optimization problem of minimizing $c(k)$ over the allowed values of $k$. One can find an optimal $k^*$ either analytically or numerically.

This raises two fundamental questions. The first is how frequently to perform block replacements, i.e., how to determine $k^*$. The second, and subtler, question is what information should be used when making the replacement decision. At one extreme, the operator tracks only the time elapsed since the last block replacement, which is a single scalar. At the other extreme, the operator tracks the current age of every individual item, which is an $N$-dimensional vector.

\textbf{Time-elapsed policy.} We formulate a Markov decision process (MDP) in which the state is the time $n$ elapsed since the last block replacement, and the decision at each period is whether to perform a block replacement or to continue for one more period, assuming that $F$ is known. We show that an optimal policy is a threshold policy at $k_s = k^*$ (Section~\ref{sec:time-elapsed-mdp}). Crucially, this MDP has a state space of cardinality $K+1$ regardless of fleet size $N$, and directly grounds the data-driven learning algorithms described below.

\textbf{Age-vector policy.} We also formulate an MDP in which the state is the current age of each item and the decision at each period is whether to replace all items now or to continue for one more period (in which case individual failures are replaced instantaneously), assuming $F$ is known. We show that the optimal policy has a monotone threshold structure under increasing failure rate (IFR) distributions (Theorem~\ref{thm:mdp_structure}), and that the MDP optimum can outperform $c(k^*)$ by as much as 50\% (Section~\ref{sec:numeric}). Crucially, the state space grows as $(K+1)^N$, making this MDP computationally intractable for large fleets; it nonetheless serves as a gold-standard benchmark for evaluating $\mathrm{BR}(k^*)$ and quantifying the structural cost gap $\delta_{\mathrm{struct}} = c(k^*) - g^*_{\mathrm{age}}$ between block and age-dependent replacement.

We now address another important issue that arises during implementation. As mentioned before, the value of $k^*$ depends on the lifetime distribution $F$. In practice, however, $F$ is often unknown. We assume that we can keep records of when the unplanned and block replacements have occurred so far. We begin with $N$ new items at time zero and schedule the first block replacement at time $k_1$, chosen in some reasonable fashion. At the $n$-th block replacement time, we use the data collected so far to determine the next inter-replacement time $k_{n+1}$. Such a policy is called a \textit{data-driven block replacement policy}. We study several such policies in this paper.

These data-driven policies can be broadly classified as one-step or two-step policies. One-step policies use the observed data to directly estimate $c(k)$ and then choose the next $k$ as a minimizer of this estimated $c(k)$. Two-step policies use the current data to estimate the lifetime distribution $F$, then use this estimated $F$ to compute $c(k)$ and select the next $k$ as its minimizer. We study both types. To keep the analysis tractable, we assume lifetimes are positive integer-valued random variables.

\textbf{(i)~One-Step Policy: Bandit Algorithms.} We assume there are a finite number of choices for $k$, say $\{1, 2, \ldots, K\}$. We think of each choice as a bandit arm and follow the lower-confidence-bound (LCB) algorithm to select a new arm at each decision point. Suppose we follow this policy for $T$ decision points. We show that such a policy achieves $O((K - 1) \log T)$ cumulative regret (Section~\ref{sec:bandit}), matching the \citet{LaiRobbins} lower bound. We improve this policy by exploiting a special feature of the problem: choosing interval $k$ simultaneously reveals the costs (that would have been incurred) of all shorter intervals $j \leq k$. This leads to \emph{bandit algorithms with correlated arms} (Section~\ref{sec:corr_bandit}) that require only $O(1)$ direct pulls of suboptimal arms $k < k^*$, improving the regret bound from $O((K-1)\log T)$ to $O((K - k^*)\log T)$.

\textbf{(ii)~Two-Step Policy: Kaplan--Meier Estimation.} At each block replacement we have $N$ machines that are replaced before they fail, producing $N$ right-censored observations. We also observe a number of failures since the last block replacement, each yielding an uncensored observation. We use the Kaplan--Meier procedure \citet{Kaplan1958} to estimate $F$ from this mix of censored and uncensored observations (Section~\ref{sec:Ren+KM}). We use the estimated $F$ to estimate $c(k)$ and select $\hat{k}^*$ as its minimizer. We prove that the estimated cost converges almost surely to $c(k^*)$ and the estimated optimal policy converges to $k^*$ (Theorem~\ref{thm:kmConvergence}). Numerically, this two-step procedure consistently outperforms the one-step bandit algorithms, albeit at higher computational cost per cycle.

We view this work as a contribution to the reliability and maintenance literature that employs sequential learning tools, rather than a methodological contribution to bandit or learning theory in its own right. This paper is organized as follows. Section~\ref{sec:litRev} reviews related work. Section~\ref{sec:setup} presents the model, regret framework, and both the MDP formulations.  Section~\ref{sec:algos} develops the bandit and Kaplan--Meier algorithms.  Section~\ref{sec:numeric} reports numerical experiments.  Section~\ref{sec:conclusion} states conclusions and directions for future research. Pseudocodes for all seven algorithms are collected in the Online Appendix.

\section{Literature Review}\label{sec:litRev}

Multiple streams of literature are relevant to our work. We discuss them systematically below. We focus on the streams most directly relevant to our contributions.

\paragraph*{Classical block replacement.}
The analytical foundations of block replacement were laid by \citet{Barlow1960}, who derived optimal replacement intervals under parametric lifetime distributions. \citet{Berg1976} extended the framework to systems with increasing running costs.  The close relationship between block replacement costs and the renewal function, including the subtlety that IFR does not imply convexity of the renewal function, is examined in detail by \citet{ifrConvex}. For multi-unit systems, \citet{SheuJhang1996} provide explicit cost formulae for group replacement with minimal repair. \citet{ChoParlar1991} survey multi-unit maintenance models, and \citet{Dekker1997} focus specifically on economic dependence, the cost structure (simultaneous replacement is cheaper per machine than individual replacement) that motivates block replacement in this paper.

\paragraph*{MDP formulations and structural results.} \citet{KurtKharoufeh2010} established threshold policies for a single machine with Markovian degradation, showing that the optimal replacement region is characterised by a state threshold.  \citet{Andersen2022} compared MDP algorithms for multi-component replacement problems numerically, illustrating the computational challenges that motivate scalable data-driven alternatives.  The theoretical foundations for average-cost MDPs used in Section~\ref{sec:MDP}, including Bellman optimality and convergence of value iteration, are developed in \citet{Puterman} and \citet{Ross}.

\paragraph*{Adaptive and data-driven maintenance.}
The problem of learning maintenance policies from operational data has grown substantially in importance.  \citet{Corman2017} used Kaplan--Meier estimates of failure time data to calibrate preventive maintenance schedules for railway braking systems, and \citet{Deprez2021} combined survival regression with renewal theory to price maintenance contracts.  Most directly related to the present paper is \citet{Puyao}, who formulate single-machine maintenance as a renewal-reward bandit and develop data-driven replacement scheduling algorithms.  Our work extends this framework to block replacement policies, introduces variance-adaptive and correlated-arm LCB variants, demonstrates that our algorithms match the Lai–Robbins lower bound, and adds the Kaplan--Meier renewal approach.

\paragraph*{Multi-armed bandit algorithms.} \citet{LaiRobbins} established the fundamental $\Omega(\log T)$ lower bound on regret for any consistent algorithm; \citet{Auer2002} proved that UCB achieves $O(\log T)$ regret for bounded rewards.  \citet{Audibert2009} introduced variance-adaptive algorithms based on the empirical Bernstein inequality, yielding tighter instance-dependent constants when arm reward variances differ; we adapt these for our Bernstein-based LCB algorithms. Side observations and correlated arm structures were studied by \citet{Mannor2011} and \citet{Caron2012}. The cascade bandit of \citet{Kveton2015}, where a single play reveals a cascade of outcomes, is the closest structural analogue to our nested observation property. A comprehensive survey of bandit algorithms is given in \citet{Bubeck2012}. 

\paragraph*{Nonparametric renewal estimation under censoring.} The Kaplan--Meier estimator \citet{Kaplan1958} is the nonparametric maximum likelihood estimator for survival functions under right censoring. \citet{Vardi1982} first studied nonparametric estimation in renewal processes from multiple independent streams; \citet{Grubel1993} constructed an empirical renewal function estimator for complete data and proved $\sqrt{n}$-consistency.  Under censoring, \citet{Baxter1995} showed that the KM plug-in estimator of the renewal function is consistent under mild conditions.  Our KM algorithm in Section~\ref{sec:Ren+KM} builds directly on these results.

Our paper adapts and extends these existing methodologies to the specific problem of deriving optimal block replacement policies. To the best of our knowledge, this work is novel and has not appeared in literature before. 

\section{Model Specification}\label{sec:setup}

\subsection{Model Description}

Consider a system of $N$ identical machines operating independently of each other. Each machine $j \in \{1, \cdots, N\}$ has independent, identically distributed (possibly unbounded) integer-valued lifetimes $\left(L^j_i \right)_{i \geq 1}$, where $L^j_i$ represents the $i\textsuperscript{th}$ lifetime of machine $j$. Let $L$ denote a generic random variable with the same common distribution with probability mass function $f: \mathbb{N}_+ \to [0,1]$ and cumulative distribution function $F: \mathbb{N}_+ \to [0,1]$, where $F(t) = \sum \limits_{s=1}^t f(s) = \mathbb{P}(L \leq t)$ for $t = 1, 2, \cdots$.

We impose the following structural assumptions about the lifetimes:
\begin{enumerate}[label = (\roman*)]
    \item The lifetime sequences $\{L^j_i : i \geq 1\}$ for different machines $j = 1, \ldots, N$ are mutually independent and identically distributed taking values on $\{1, 2, 3, \cdots\}$. This is without essential loss of generality, as continuous lifetimes can be discretized to match any reporting resolution. 
    \item The lifetime distribution has finite mean $\mu := \mathbb{E}[L] = \sum \limits_{t=1}^\infty t \cdot f(t) < \infty$.
\end{enumerate}

The system operates under a \textit{block replacement policy} that combines scheduled block replacements with individual replacements immediately upon failure. In block replacements, all $N$ machines are replaced simultaneously at a fixed cost $c_b > 0$, regardless of their operational status. When a machine fails between block replacements, it is immediately replaced at cost $c_f > 0$. If a machine fails exactly at a scheduled block replacement time, such a failure is recorded as an individual replacement at cost $c_f$ and the block replacement at cost $c_b$ still proceeds as scheduled.

All replacements (block and individual) are assumed to occur instantaneously. Upon replacement, a machine begins a new lifetime with distribution $F$. Machines can be replaced at any discrete time instant regardless of its operational status. We assume $c_f > c_b/N$, which ensures that individual replacements are more expensive per machine than block replacements, thereby incentivizing the use of block replacement policies. This cost structure incentivizes proactive block replacement but must be balanced against unnecessary replacement of functional machines.

Let $\tau_0 = 0$ and let $\tau_t$ denote the $t\textsuperscript{th}$ \textit{block replacement time} for $t \geq 1$. Define the \textit{inter-block replacement interval} as $k_t := \tau_t - \tau_{t - 1}$ for $t \geq 1$. The $t\textsuperscript{th}$ \textit{cycle} of the system spans the time instants $\{\tau_{t - 1} + 1, \ldots, \tau_t\}$ for $t \geq 1$. For each machine $j \in \{1, \ldots, N\}$, let $Y_t^j$ denote the (random) number of failures (and hence individual replacements) of machine $j$ during the $t\textsuperscript{th}$ cycle. The total cost incurred during the $t$-th cycle is $\psi_t = c_b + c_f \cdot \sum \limits_{j=1}^N Y_t^j$. For $t \geq 1$, the cost per unit time incurred during the $t\textsuperscript{th}$ cycle is given by $\gamma_t = \frac{\psi_t}{k_t} = \frac{c_b + c_f \sum \limits_{j=1}^N Y_t^j}{k_t}$. 

Let $(N(t))_{t \geq 0}$ denote the renewal process with interarrival distribution $F$, where $N(t)$ represents the number of renewals (failures) in the interval $[0, t]$. The associated \textit{renewal function} is defined as $M(t) := \mathbb{E}[N(t)]$. As shown in \citet{Kulkarni}, $M(t)$ satisfies the \textit{renewal equation}:
\begin{equation}\label{eqn:renewal}
    M(t) = F(t) + \sum \limits_{s=1}^t f(s) M(t - s),
\end{equation}
with initial condition $M(0) = 0$. 

A \textit{block replacement policy} with parameter $k \in \mathbb{N}_+$, denoted $BR(k)$, is a stationary policy under which all machines are replaced every $k$ time units, regardless of their operational status. Between block replacements, machines are replaced individually upon failure. Under $BR(k)$, the system operates in cycles of deterministic length $k$. By the renewal reward theorem (see \citet{Ross}), the long-run average cost per unit time is given by $c(k) = \frac{\mathbb{E}[\text{Cost in one cycle}]}{\mathbb{E}[\text{Cycle length}]} = \frac{c_b + c_f \cdot N \cdot M(k)}{k}$, because the expected number of failures for a single machine over the interval $[0, k]$ is $M(k)$, and we have $N$ independent machines.

The set of feasible replacement intervals is $[K] = \{1, 2, \ldots, K\}$, where $K \in \mathbb{N}_+$ is a known upper bound on block replacement intervals. This bound may arise from operational constraints like maximum time between scheduled maintenance. Further we assume that $f(k) > 0$ for all $k \in [K]$, that is, the machines can fail after reaching age $k$ for any $k \in [K]$. The \textit{optimal block replacement interval} is defined as $k^* := \argmin_{k \in [K]} c(k)$. The corresponding \textit{optimal long-run average cost} is $c^* := c(k^*)$. For simplicity, we assume that $k^*$ is unique, that is, $c(k^*) < c(k)$ for all $k \neq k^*$. This assumption does not affect the analytic interpretability of the results that follow. \\

\begin{remark}
    The cost function $c(\cdot)$ need not be convex in general (see Section~\ref{sec:numeric}). Convexity of the renewal function $M(\cdot)$ is sufficient to ensure convexity of $c(\cdot)$, but unfortunately, the convexity of the renewal function $M(\cdot)$ is not characterized by simple properties of the lifetime distribution $F$. In particular, \textit{increasing failure rate} (IFR) of $L$ does not imply convexity of $M(\cdot)$. See \citet{ifrConvex} for a detailed discussion on this topic. \\
\end{remark} 

\subsection{Regret Framework for Online Learning}\label{sec:regret}

In practical applications, the lifetime distribution $F$ (and hence the renewal function $M$ and optimal policy $k^*$) is typically unknown. We must learn the optimal policy from observed data while the system operates. The unknown $F$ motivates an online learning framework.

Consider an algorithm that produces a sequence of replacement interval decisions $\{k_t\}_{t \geq 1}$, where $k_t \in [K]$ is the block replacement interval chosen for the $t$\textsuperscript{th} cycle. For $k \in [K]$, let $\Delta_k := c(k) - c(k^*)$ denote the suboptimality gap. The performance of the algorithm is evaluated by the \textit{cumulative regret} over a horizon of $T$ cycles, defined as
\begin{equation}\label{eqn:regDecomp}
    \operatorname{Reg}(T) 
  := \mathbb{E}[Z(T)] - T \cdot c(k^*)
   = \sum \limits_{t=1}^T \E[\Delta_{k_t}],
\end{equation}
where $Z(T) := \sum \limits_{t=1}^T \gamma_t$ is the sum of per-unit-time costs incurred during the first $T$ cycles. This measures the expected excess per-unit-time cost compared to an oracle that always uses the optimal policy $k^*$. We also introduce an alternative formulation that is more interpretable in terms of total cost incurred. Define
\begin{equation}\label{eqn:regTilDecomp}
    \widetilde{\operatorname{Reg}}(T)
  := \mathbb{E}\left[\widetilde{Z}(T)\right] - c(k^*) S(T)
   = \sum \limits_{t=1}^T k_t \E[\Delta_{k_t}],
\end{equation}
where $\widetilde{Z}(T) = \sum \limits_{t=1}^T \psi_t$ is the total cost incurred, and $S(T) = \sum \limits_{t=1}^T k_t$ is the total elapsed time, over the first $T$ cycles. Since $1 \leq k_t \leq K$, the two versions of regret are related by
\begin{equation}\label{eq:compare_reg}
    \text{Reg}(T) \leq \widetilde{\text{Reg}}(T) \leq K \cdot \text{Reg}(T),
\end{equation}
so both grow at the same asymptotic rate in $T$. We primarily use $\text{Reg}(T)$ for our theoretical analysis, while both versions are compared in our numerical studies.

Regret is analysed by its asymptotic growth rate. We write $\text{Reg}(T) = O(f(T))$ if there exist positive constants $C$ and $T_0$ such that for all $T \geq T_0$, we have $\text{Reg}(T) \leq C f(T)$. Similarly, for two functions $f, g : \N \to \R^+$, we say that $g$ is $\Omega(f)$ if there exist constants $c > 0$ and $n_0 \in \N$ such that $g(n) \geq c \cdot f(n)$ for all $n \geq n_0$. 

Before presenting our algorithms, we recall fundamental limits on achievable regret due to \citet{LaiRobbins}. Let $P_k$ be the distribution of the observed cost under policy $BR(k)$. With $D_{\mathrm{KL}} (\cdot \, || \cdot)$ denoting the Kullback-Leibler divergence between two distributions, Lai \& Robbins \citet{LaiRobbins} show that for any ``consistent'' algorithm,
\begin{equation}
    \liminf_{T \to \infty} \frac{\mathbb{E}[\text{Reg}(T)]}{\log T} \geq \sum \limits_{\substack{k \in [K]:\\ k \neq k^*}} \frac{\Delta_k}{D_{\mathrm{KL}}(P_k \| P_{k^*})}.
\end{equation}

This says $\text{Reg}(T) = \Omega(\log T)$ and hence by (\ref{eq:compare_reg}), $\widetilde{\text{Reg}}(T) = \Omega(\log T)$. The key insight is that distinguishing policy $k$ from $k^*$ requires $\Omega(\log T / \mathrm{KL}(P_k \| P_{k^*}))$ observations, so the best achievable regret grows as $O(\log T)$. \\

\begin{remark}
For any sequence $\{k_t\}_{t \geq 1}$ of decisions, we have (i) $\frac{\mathbb{E}[\widetilde{Z}(T)]}{\mathbb{E}[S(T)]} \geq c(k^*)$ because $\mathbb{E}[\widetilde{Z}(T)] = \sum \limits_{t=1}^T \mathbb{E}[k_t c(k_t)] \geq c(k^*)\,\mathbb{E}[S(T)]$ and (ii) $\frac{\mathbb{E}[Z(T)]}{T} \geq c(k^*)$ because $\mathbb{E}[Z(T)] = \sum \limits_{t=1}^T c(k_t) \geq T \cdot c(k^*)$. 
\end{remark}

\subsection{Time-elapsed MDP Formulation}\label{sec:time-elapsed-mdp}

This Markov Decision Process (MDP) uses the time elapsed since the last block replacement as the state variable. This yields a state space of cardinality $K+1$, independent of fleet size $N$. We assume the lifetime distribution $F$ is known. 

Let $k \in [K]$ denote the number of time periods elapsed in the current block replacement cycle, and let $k = 0$ denote the state immediately following a block replacement. Define the renewal mass function $m(k) := M(k) - M(k-1)$ for $k \geq 1$, where $M(k) = \mathbb{E}[N(k)]$ is the renewal function of Section~\ref{sec:setup}. The quantity $m(k)$ equals the expected number of failures per machine during the $k^{\text{th}}$ period of a renewal process in which each machine is replaced immediately upon failure; the expected total number of failures across the fleet of $N$ independent machines during period $k$ is $Nm(k)$.

The action space is $\mathcal{A} = \{0,1\}$; at each state $k \in [K]$, the operator selects one of two actions: continue ($a = 0$), incurring expected cost $c_fNm(k)$ from individual failure replacements and transitioning to state $k+1$; or block replace ($a = 1$), incurring cost $c_b$ and transitioning to the reset state $k = 0$, from which the next cycle begins immediately at state $k = 1$. Block replacement is mandatory at $k = K$ since $K$ is the upper bound on feasible replacement intervals. The expected one-step costs are $c(0,k) = c_fNm(k)$ and $c(1,k) = c_b$.

We adopt the undiscounted average cost criterion for this infinite-horizon problem. Under a stationary deterministic policy $\pi:[K]\to\{0,1\}$, the long-run average cost per unit time starting from state $k$ is
\begin{equation}
l^{\pi}(k) = \limsup_{T\to\infty} \frac{1}{T}\,\mathbb{E}^{\pi}\left[\sum_{t=0}^{T-1} c(\pi(k_t),\, k_t) \,\middle|\, k_0 = k\right].
\label{eq:te-avgcost}
\end{equation}
The optimal average cost is $l^* = \inf_{\pi} l^{\pi}(k)$. As shown in \citet{Ross}, the optimal long-run average cost $l^*$ and the relative value function $v:\{0,\ldots,K\}\to\mathbb{R}$, normalized so that $v(0) = 0$, satisfy the average-cost Bellman optimality equation for $k \in [K]$ given by 
\begin{equation}
l^* + v(k) = \min\bigl(c_fNm(k) + v(k+1), c_b + v(0)\bigr).
\label{eq:bellman-elapsed}
\end{equation}
An optimal policy $\pi^*$ achieves the minimum in~\eqref{eq:bellman-elapsed} at every state. Algorithm~\ref{alg:vi-te} estimates $l^*$ and $v$ by iterating the Bellman operator. We set $k = 0$ as the reference state, so $w^{(n)}(0) = 0$ for all $n$, and monitor convergence via the span seminorm $\mathrm{sp}(d^{(n)}) := \max_k d^{(n)}(k) - \min_k d^{(n)}(k)$ where $d^{(n)}(k) := w^{(n+1)}(k) - w^{(n)}(k)$.

\begin{theorem}[\textit{Convergence of Value Iteration}]
\label{thm:te-vi-convergence}
The time-elapsed MDP satisfies the unichain condition: state $0$ is accessible from every state $k \in [K]$ under the block replacement action. Consequently, the value iteration sequence $\{w^{(n)}\}$ generated by Algorithm~\ref{alg:vi-te} converges, and upon termination with $\mathrm{sp}(d^{(n)}) < \varepsilon$, the following hold: (i) the midpoint estimate satisfies $|\hat{l}^* - l^*| \leq \varepsilon/2$; (ii) the computed policy $\hat{\pi}^*$ is $\varepsilon/2$-optimal.
\end{theorem}

\proof{Proof.}
This is a standard result for unichain average-cost MDPs; see \citet{Puterman}, Theorem~8.5.6. The unichain condition holds because the block replacement action ($a=1$) transitions any state directly to state $0$. \hfill$\blacksquare$

A stationary deterministic policy $\pi$ is characterized by its replacement set $S_\pi := \{k:\pi(k) = 1\}$. Since the state resets to $0$ upon any block replacement and then advances deterministically through $1, 2,\ldots$, each cycle under $\pi$ runs from state $0$ until the system first enters $S_\pi$. States beyond $\min(S_\pi)$ are never visited within a single cycle: once the system reaches $\min(S_\pi)$, it replaces and restarts. Consequently, if $\min(S_\pi) = k$, the cycles are i.i.d.\ with deterministic length $k$ and expected cost $c_b + c_fNM(k)$, so the long-run average cost is $l^{\pi} = c(k)$ by the renewal reward theorem. In particular, $l^{\pi}$ depends on $\pi$ only through $\min(S_\pi)$.

\begin{theorem}\label{thm:time-mdp}
For every lifetime distribution $F$ with $\mathbb{E}[L] < \infty$, the following hold for the time-elapsed MDP:
\begin{enumerate}[label=(\roman*)]
\item The optimal average cost is $l^* = c(k^*)$.
\item The threshold policy $\pi_{k^*}(k) = \mathbf{1}[k \geq k^*]$ is an optimal stationary policy.
\item The threshold $k_s := \min\bigl\{k \in [K] : c_fNm(k) + v(k+1) > c_b + v(0)\bigr\} = k^*$. 
\end{enumerate}
\end{theorem}

\proof{Proof.}
Note that under any stationary deterministic policy $\pi$, the system resets to state $0$ upon every block replacement and then advances through states $1, 2, \cdots$ deterministically until it first enters $S_\pi$. Each cycle therefore has deterministic length $\min(S_\pi)$ and expected cost $c_b + c_f \cdot N M(\min(S_\pi))$, since the expected number of individual failures per machine over an interval of length $\min(S_\pi)$ is $M(\min(S_\pi))$ by definition of the renewal function. The renewal reward theorem (Ross 1996) then gives $l^\pi = \frac{c_b + c_f \cdot N M(\min(S_\pi))}{\min(S_\pi)} = c(\min(S_\pi))$. Since $l^{\pi} = c(\min(S_\pi))$ for every stationary deterministic policy $\pi$, optimizing over all such policies reduces to minimizing $c(k)$ over $k \in [K]$. The infimum is therefore $c(k^*)$, achieved by any policy with $\min(S_\pi) = k^*$; the threshold policy $\pi_{k^*}$ is one such policy. By Theorem~8.5.4 of \citet{Puterman}, there exists an optimal stationary deterministic policy, so $l^* = c(k^*)$ and $\pi_{k^*}$ is optimal. For part~(iii): under $\pi_{k^*}$ with $l^* = c(k^*)$, the Bellman equation~\eqref{eq:bellman-elapsed} gives $c_fNm(k) + v(k+1) \leq c_b + v(0)$ for all $k < k^*$ (continue is optimal) and $c_b + v(0) \leq c_fNm(k^*) + v(k^*+1)$ (block replace is optimal at $k^*$), so $k_s = k^*$. \hfill$\blacksquare$

While the time-elapsed MDP achieves the optimal cost $c(k^*)$ within the class of block replacement policies, it discards information about the current ages of individual machines. A richer formulation that tracks the full age vector of the fleet can exploit this information to achieve a strictly lower average cost under IFR lifetime distributions. We develop this formulation next.

\subsection{Age-vector MDP Formulation}\label{sec:MDP}

This MDP approach models the system state explicitly by tracking the ages of all machines, enabling the computation of optimal policies. This framework allows for state-dependent policies that can exploit information about the current operational ages of individual machines. 

Throughout this subsection, we assume that the lifetime distribution $F$ is known. This MDP formulation serves as a benchmark to evaluate the performance gap between block replacement policies and fully state-aware policies. In Section~\ref{sec:numeric}, we  numerically compare the costs achieved by the proposed bandit (Section~\ref{sec:bandit}) and Kaplan - Meier (Section~\ref{sec:Ren+KM}) algorithms (which learn without knowing $F$) with this benchmark.

Let the system state at the beginning of time step $n$ be $\mathbf{s}_n = (s_n^1, \ldots, s_n^N) \in \mathcal{S}$, where $s_n^j \in \{0, 1, \ldots, K\}$ represents the current age (time since last replacement) of machine $j$. The state space is $\mathcal{S} = \{0, 1, \ldots, K\}^N$, with cardinality $|\mathcal{S}| = (K+1)^N$. Let $S(t) = 1 - F(t)$ denote the survival function, $h(t) := \frac{f(t)}{S(t-1)} = P(L = t \mid L \geq t)$ the discrete hazard function, and $\bar{h}(t) := 1 - h(t) = S(t)/S(t-1)$ be the survival ratio for $t \geq 1$. We impose the boundary condition $h(K+1) = 1$ so that $\bar{h}(K+1) = 0$. This means that a machine which has survived $K$ units of time since its last replacement fails with certainty at the next step; age $K+1$ is therefore never reached in the state space. %This is without loss of generality because if $P(L > K) > 0$ for the true distribution, one may define $f(K+1) := P(L > K)$ (placing all remaining mass at $K+1$) and set $h(K+1) = 1$ accordingly, which changes $c(k)$ only for $k = K+1 \notin [K]$ and leaves the optimization problem unchanged.

The action space is $\mathcal{A} = \{0, 1\}$, where action $a = 0$ denotes continuation (individual machines replaced on failure) and $a = 1$ denotes block replacement of all machines, transitioning the system to $\mathbf{0} := (0, \ldots, 0)$. Under action $a = 0$, each machine $j$ at age $s^j$ fails over the next period independently with probability $h(s^j + 1)$ and survives with probability $\bar{h}(s^j + 1)$. A failed machine is immediately replaced (age resets to $0$); a surviving machine ages by one unit to $s^j + 1$. Since $h(K + 1) = 1$, if machine $j$ survives, then $s^j \leq K-1$. The transition therefore never produces an age exceeding $K$. 

Since machines fail independently, the set of machines that fail over one time step is a random subset $A \subseteq [N]$. For each subset $A \subseteq [N] := \{1, 2, \cdots, N\}$ of machines that fail, the resulting next state $s^{[A]} \in \mathcal{S}$ is defined by
\begin{equation}
  \label{eq:transition}
  s^{[A],j} :=
  \begin{cases}
    0       & j \in A, \\
    s^j + 1 & j \notin A.
  \end{cases}
\end{equation}

For $A \subseteq [N]$, the one-step transition probabilities under action $a = 0$ are $P_0\bigl(s,\, s^{[A]}\bigr) = \prod_{j \in A} h(s^j{+}1)\cdot \prod_{j \notin A} \bar{h}(s^j{+}1)$, with $P_0(s, s') = 0$ if $s'$ is not of the form $s^{[A]}$ for any $A \subseteq [N]$. Under action $a = 1$ (block replacement), the transition is deterministic: $P_1({s}, {s}') = \mathbf{1}_{s' = \mathbf{0}}$. 

The expected one-step costs are given by $c(0, s) = c_f \sum \limits_{j=1}^N h(s^j + 1)$ and $c(1, s) = c_b$. The cost under $a = 0$ follows because machine $j$ fails independently with probability $h(s^j + 1)$, and by linearity of expectations, the expected number of failures across $N$ machines is $\sum \limits_{j = 1}^N h(s^j + 1)$. 

We adopt the undiscounted average cost criterion for this infinite horizon problem. Under a stationary deterministic policy $\pi: \mathcal{S} \to \mathcal{A}$, the long-run average cost per unit time starting from state $s$ is defined as
\begin{equation}\label{eq:avg-cost}
g^\pi(s) = \limsup_{T \to \infty} \frac{1}{T} \mathbb{E}^\pi\left[\sum \limits_{t=0}^{T-1} c(\pi(s_t), s_t) \middle| s_0 = s\right].
\end{equation}
where the expectation is taken over state transitions induced by policy $\pi$.

The optimal average cost is $g^* = \inf_\pi g^\pi({s})$, and a policy $\pi^*$ is optimal if $g^{\pi^*}(\mathbf{s}) = g^*$ for all $\mathbf{s} \in \mathcal{S}$. As proved in \citet{Ross}, the optimal average cost $g^*$ and relative value function $w^* : \mathcal{S} \to \mathbb{R}$ satisfy the Bellman optimality equation for each $s \in \mathcal{S}$, given by
\begin{equation}\label{eq:bellman}
  g^* + w^*(s) = \min_{a \in \mathcal{A}} 
  \Bigl\{c(a, s) + \sum \limits_{s' \in \mathcal{S}} P_a(s, s') w^*(s')\Bigr\}.
\end{equation}
The relative value function captures the advantage (or disadvantage) of starting in state s versus the reference state, in terms of accumulated cost before the process regenerates. An optimal policy $\pi^*$ achieves the minimum in \eqref{eq:bellman} at every state. Algorithm~\ref{alg:value_iteration_avg} describes value iteration to estimate $g^*$ and $w^*$ by iterating the Bellman operator. We set $\mathbf{s}_{\text{ref}} = \mathbf{0}$ is the reference state (so $w^{(k)}(\mathbf{0}) = 0$ for all $k$). The convergence criterion is using the span seminorm $\mathrm{sp}(d^{(k)}) := \max_s d^{(k)}(s) - \min_s d^{(k)}(s)$, where $d^{(k)}(s) := w^{(k+1)}(s) - w^{(k)}(s)$, which is standard for average cost MDPs. The reference state $\mathbf{s}_{\text{ref}} = \mathbf{0}$ is arbitrary; any state can be used since the relative value function is unique only up to an additive constant.\\

\begin{theorem}[\textit{Convergence of Value Iteration}]
\label{thm:value_iter_convergence}
    The MDP described above satisfies the \textit{unichain condition}, that is, there exists a state that is accessible from every other state under every stationary policy. Consequently, the value iteration sequence $\{w^{(k)}\}$ generated by Algorithm~\ref{alg:value_iteration_avg} converges, and upon termination with $\mathrm{sp}(d^{(k)}) < \varepsilon$, the following hold:
\begin{enumerate}[label=(\roman*)]
    \item the midpoint estimate satisfies $|\widehat{g}^* - g^*| \leq \varepsilon/2$;
    \item the computed policy $\widehat{\pi}^*$ is $\varepsilon/2$-optimal,
      i.e., $g^{\widehat{\pi}^*} - g^* \leq \varepsilon/2$.
  \end{enumerate}
\end{theorem}

\proof{Proof.} This is a standard result in MDP literature (for instance, see \citet{Puterman}). The unichain condition holds because state $\mathbf{0}$ is accessible from any state $\mathbf{s} \in \mathcal{S}$ because the $f(k) > 0$ for all $k \in [K]$ and so the probability of all machines failing at a given time is positive. \hfill $\blacksquare$ 

\begin{remark}
    The average-cost optimality equation~\eqref{eq:bellman} holds over all policies, including history-dependent ones. The unichain property (Theorem~\ref{thm:value_iter_convergence}) then guarantees that the infimum long run average cost is achieved by a stationary Markovian policy $\pi^*$ (see \citet{Ross}). 
\end{remark}

The optimal policy for the MDP formulation possesses several structural properties which are summarized in the following theorem. Equip $\mathcal{S}$ with the componentwise partial order as ${s} \succeq {s}'$ if $s^j \geq s'^j$ for all $j \in [N]$. \\

\begin{theorem}\label{thm:mdp_structure}
    Let $\pi^*$ denote an optimal policy for the MDP. Suppose the lifetime distribution $F$ has increasing failure rate (IFR), that is, $h(t)$ is non-decreasing in $t$. Then the following are true:
\begin{enumerate}[label = (\alph*)]
    \item \textbf{Value Function Monotonicity}: The optimal relative value function $w^*$ is non-decreasing in each coordinate, that is, $w^*({s}') \geq w^*({s})$ whenever ${s}' \succeq {s}$.
    \item \textbf{Policy Monotonicity}: If $\pi^*(\mathbf{s}) = 1$ and $\mathbf{s}' \succeq \mathbf{s}$, then $\pi^*(\mathbf{s}') = 1$.
    
    \item \textbf{Symmetry}: The optimal relative value function and optimal policy are invariant under permutation of machine indices, that is, for any permutation $\sigma$ of $[N]$, $w^*\left(s^1, \ldots, s^N\right) = w^*\left(s^{\sigma(1)}, \ldots, s^{\sigma(N)}\right)$ and $\pi^*\left(s^1, \ldots, s^N\right) = \pi^*\left(s^{\sigma(1)}, \ldots, s^{\sigma(N)}\right)$. 
    
    \item \textbf{Threshold Structure}: For each fixed $i \in [N]$ and ${s}_{-i} \in \{0,\ldots,K\}^{N-1}$,  there exists a threshold $\tau^*_i({s}_{-i}) \in \{0,\ldots,K\} \cup \{\infty\}$ such that $\pi^*(s^i, {s}_{-i}) = \mathbf{1}_{s^i \geq \tau^*_i({s}_{-i})}$, where we set $\tau^*_i({s}_{-i}) = \infty$ if $\pi^*(\cdot, {s}_{-i}) \equiv 0$. Moreover, $\tau^*_i({s}_{-i})$ is non-increasing in each component of ${s}_{-i}$.
\end{enumerate}
\end{theorem}

The proof is in appendix~\ref{app:thm2}. 

\begin{remark}
    Theorem~\ref{thm:mdp_structure} establishes that the optimal policy is completely characterised by the $N$ threshold functions $\{\tau^*_i(s_{-i})\}_{i \in [N]}$, requiring $N(K+1)^{N-1}$ values in total, rather than a binary decision at each of the $(K+1)^N$ states. By the symmetry of part~(c), $\tau^*_i(s_{-i})$ depends only on the multiset of ages $\{s^1, \ldots, s^N\}$, so one need only store threshold values for states with $s^1 \leq s^2 \leq \cdots \leq s^N$, of which there are $\binom{K+N}{N}$, reducing storage by a factor of up to $N!$.
\end{remark}

\begin{remark}
    The time-elapsed MDP (Section~\ref{sec:time-elapsed-mdp}) has a state space of cardinality $K+1$, requires no assumption on $F$ beyond $\mathbb{E}[L] < \infty$, is tractable for all fleet sizes $N$, and establishes that block replacement is optimal within the class of time-elapsed policies, achieving cost $g^* = c(k^*)$. The age-vector MDP of Section~\ref{sec:MDP} has a state space of cardinality $(K+1)^N$, requires IFR for its structural results, is tractable only for small $N$, and achieves the lower optimal cost $g^*_{\mathrm{age}} \leq c(k^*)$ by exploiting the joint distribution of machine ages. The structural gap $\delta_{\mathrm{struct}} = c(k^*) - g^*_{\mathrm{age}} \geq 0$ quantifies the value of age information and is strictly positive under IFR for $N \geq 2$.
\end{remark}

\section{Data-Driven Algorithms}\label{sec:algos}

\subsection{Multi-arm bandit based}\label{sec:bandit}

We formulate the block replacement problem as a stochastic multi-armed bandit (MAB) with $K$ arms, where each arm $k \in [K]$ represents a block replacement interval of length $k$. At each decision epoch $t \geq 1$, the decision maker selects a block replacement interval $k_t$,  observes the total number of failures $Y_t$ that occur across all $N$ machines during that interval, incurring a total cost of $c_b + c_f Y_t$ in that cycle. Let $\mu_k := \mathbb{E}\left[C^{(k)}\right] = c_b + c_f \cdot N \cdot M(k) = k \cdot c(k)$ for $k \in [K]$. The goal is to identify $k^* = \argmin_{k \in [K]} \frac{\mu_k}{k}$, the optimal arm that incurs the minimum cost per unit time. The per-unit-time cost $C^{(k)}/k$ has the same range (width of support) $b := c_f N$ for every arm $k \in [K]$, since the minimum is $c_b/k$ (no failures) and the maximum is $c_b/k + c_f N$ (every machine fails every period). Throughout, log denotes the natural logarithm.

\subsubsection{Independent Arms}\label{sec:indep_bandit}

In the independent arms setting, each arm $k$ is treated as an independent learning problem. We develop lower confidence bound (LCB) algorithms based on Hoeffding's and Bernstein's concentration inequalities.

\paragraph{Hoeffding-based.}

If $k$ has been pulled $n_k(t)$ times up to decision epoch $t$, with observed costs $C_1^{(k)}, \ldots, C_{n_k(t)}^{(k)}$, the empirical mean cost associated to arm $k \in [K]$ is $\widehat{c}_k(t) = \frac{1}{n_k(t)} \sum \limits_{i=1}^{n_k(t)} \frac{C_i^{(k)}}{k}$. The Hoeffding-based lower confidence bound for arm $k$ is 
\[\text{LCB}_k^{IH}(t) = \widehat{c}_k(t) - b\sqrt{\frac{2 \log(t)}{n_k(t)}}.\] 
The algorithm selects the arm with the minimum LCB, that is $k_t = \argmin_{k \in [K]} \text{LCB}_k^{IH}(t-1)$.  

\begin{theorem}\label{thm:indep_Hoeff}
    Under the Hoeffding-based algorithm (Algorithm~\ref{alg:ind_hoeffding}) with independent arms, the expected cumulative regret satisfies
\begin{equation*}
\text{Reg}(T) = O\left(\sum \limits_{k \neq k^*} \frac{b^2 \log T}{\Delta_k}\right) = O\left((K - 1) \log T\right).
\end{equation*}
\end{theorem}

\proof{Proof.}
    This is a standard result in MAB literature that follows from the arm-wise decomposition of regret given in  Equation~\ref{eqn:regDecomp} and the fact that $\E[n_k(T)] = O(\log T)$ for $k \neq k^*$ (see \citet{Lattimore}). \hfill $\blacksquare$

\paragraph{Bernstein-based.}

Bernstein's inequality provides tighter confidence bounds (smaller exploration bonus) by exploiting variance information. For arm $k \in [K]$ with $n_k(t) \geq 2$ samples, we compute the empirical variance as \[\widehat{\sigma}_k^2(t) = \frac{1}{n_k(t)-1} \sum \limits_{i=1}^{n_k(t)} \left(\frac{C_i^{(k)}}{k} - \widehat{c}_k(t)\right)^2.\] 
The Bernstein-based LCB is constructed as follows: when $n_k(t) = 1$, define $\mathrm{LCB}_k^{IB}(t) = \widehat{c}_k(t) - b\sqrt{2\log t\,/n_k(t)}$; when $n_k(t) \geq 2$, set 
\begin{align}\label{eq:bern-lcb-indep}
\mathrm{LCB}_k^{IB}(t) =
\widehat{c}_k(t) &- \sqrt{2.4\,\widehat{\sigma}_k^2(t)\log t\,/n_k(t)} \nonumber\\
    &-3.6b\log t\,/n_k(t)
\end{align}
The specific constants 2.4 and 3.6 are chosen to match the analysis carried out in \citet{Audibert2009}. The algorithm proceeds very similar to Algorithm \ref{alg:ind_hoeffding}, but uses $\text{LCB}_k^{IB}(t)$ instead of $\text{LCB}_k^{IH}(t)$.\\

\begin{theorem}\label{thm:indep_Bern}
    Under the Bernstein-based algorithm (Algorithm~\ref{alg:ind_bernstein}) with independent arms, the expected cumulative regret satisfies
\begin{equation*}
\text{Reg}(T) = O\left(\sum \limits_{k \neq k^*} \frac{b^2 \log T}{\Delta_k}\right) = O\left((K - 1) \log T\right).
\end{equation*}
\end{theorem}

\proof{Proof.} See Theorem 4 of \citet{Audibert2009}. \hfill $\blacksquare$

\subsubsection{Correlated Arms}\label{sec:corr_bandit}

A key structural property of the block replacement problem is the nested observation structure: when selecting a block replacement interval $k$ and observing the failure history up to time $k$, we simultaneously obtain complete information about what would have happened under all smaller intervals $j < k$. This ``correlation'' among arms can be exploited to improve learning efficiency. The following proposition formalizes this observation. \\

\begin{proposition}[Nested Observations]\label{prop:nested_obs}
When arm $k$ is pulled, for each $l \leq k$ define $C^{(l)} := c_b + c_f \sum \limits_{j=1}^N N^j(l)$ where $N^j(l) := \max\left\{n \geq 0 : \textstyle\sum \limits_{i=1}^n L_i^j \leq l\right\}$.
\begin{enumerate}[label = (\alph*)]
    \item $C^{(l)}$ is determined by the observed failure history of the pull of arm $k$; 
    \item $C^{(l)} \stackrel{d}{=} \widetilde{C}^{(l)}$, where $\widetilde{C}^{(l)}$ is the cost from an independent direct pull of arm $l$.
\end{enumerate} 
\end{proposition}

\proof{Proof.} (a) Since $N^j(l) \leq N^j(k)$ a.s., $N^j(l)$ is a deterministic function of $(L_1^j,\ldots,L_{N^j(k)}^j)$, which are observed when arm $k$ is pulled. (b) The joint law of $(L_1^j,\ldots,L_{N^j(l)}^j)_{j \in [N]}$ is identical whether these lifetimes are extracted from a pull of arm $k \geq l$ or from an independent direct pull of arm $l$, so $C^{(l)}$ and $\widetilde{C}^{(l)}$ are the same function of identically distributed inputs. \hfill $\blacksquare$

Let $m_k(t)$ denote the number of samples collected for arm $k$ up to time $t$ under the correlated arms scheme. When arm $k$ is pulled at time $t$, we get $m_j(t) = m_j(t-1) + 1$ samples for all $j \leq k$ upto time $t$. The number of samples for arm $k$ equals the total number of times any arm $j \geq k$ has been pulled, that is $m_k(t) = \sum \limits_{j=k}^{K} n_j(t)$. For any arm $k$, we have $m_k(t) \geq n_k(t)$, with equality only if no arm larger than $k$ has ever been pulled. The sample counts satisfy $m_1(t) \geq m_2(t) \geq \cdots \geq m_K(t)$ for all $t$, because $m_k(t) - m_{k+1}(t) = n_k(t) \geq 0$. Smaller arms always accumulate at least as many samples as larger arms. \\

\begin{remark}
    Proposition~\ref{prop:nested_obs} justifies the simultaneous update of empirical means $\widehat{c}_j$ for all $j \leq k_t$ upon pulling arm $k_t$ because it guarantees that these derived observations are equivalent to direct pulls of arm $j$, so the effective sample count for arm $j$ is $m_j(t) = \sum \limits_{l=j}^{K} n_l(t) \geq n_j(t)$.\\ 
\end{remark}

\begin{remark}
    The regret contribution from arm $k$ is $\Delta_k \cdot \mathbb{E}[n_k(T)]$, and not $\Delta_k \cdot \mathbb{E}[m_k(T)]$, because only direct pulls of arm $k$ incur the suboptimality cost $\Delta_k$. The ``free samples'' contribute to learning the cost but not to regret, which is why the correlated algorithm achieves better regret bounds.
\end{remark}

\paragraph{Hoeffding-based.} Using the nested observations structure, we update the sample means for all arms $k \leq k_t$ when arm $k_t$ is selected using the equation
\begin{equation*}
\text{LCB}_k^{CH}(t) = \widehat{c}_k^{(m)}(t) - b\sqrt{\frac{2\log(t)}{m_k(t)}},
\end{equation*}
where $\widehat{c}_k^{(m)}(t)$ is the empirical mean based on the $m_k(t)$ samples defined by $\widehat{c}_k^{(m)}(t) = \frac{1}{m_k(t)} \sum \limits_{i=1}^{m_k(t)} \frac{C_i^{(k)}}{k}$. 

\medskip

\begin{theorem}\label{thm:corr_hoeff}
    Under the Hoeffding-based algorithm (Algorithm~\ref{alg:corr_hoeffding}) with correlated arms, the expected cumulative regret satisfies
\begin{equation*}
\text{Reg}(T) = O\left(\sum \limits_{k \neq k^*} \frac{b^2 \log T}{\Delta_k}\right) = O\left((K - k^*) \log T\right).
\end{equation*}
\end{theorem}

The proof is in Appendix~\ref{app:thm5}. 

\paragraph{Bernstein-based. }

Similar to the independent arms case, the Bernstein-based LCB with correlated arms uses $\mathrm{LCB}_k^{CB}(t) = \widehat{c}_k^{(m)}(t) - b\sqrt{\dfrac{2\log t}{m_k(t)}}$ when $m_k(t) = 1$; and when $m_k(t) \geq 2$, we set 
\begin{align}\label{eq:bern-lcb}
\mathrm{LCB}_k^{CB}(t) =
\widehat{c}_k^{(m)}(t)
  &- \sqrt{\dfrac{2.4\,\left(\widehat{\sigma}_k^{(m)}\right)^2(t)\log t}{m_k(t)}} \nonumber\\
  &- \dfrac{3.6b\log t}{m_k(t)}
\end{align}
where $\left(\widehat{\sigma}_k^{(m)}\right)^2(t) = \frac{1}{m_k(t)-1} \sum \limits_{i=1}^{m_k(t)} \left(\frac{C_i^{(k)}}{k} - \widehat{c}_k(t)\right)^2$ is the sample variance based on $m_k(t) \geq 2$. 

\medskip

\begin{theorem}\label{thm:corr_bern}
    Under the Bernstein-based algorithm (Algorithm~\ref{alg:corr_bernstein}) with correlated arms, the expected cumulative regret satisfies
\begin{equation*}
\text{Reg}(T) = O\left(\sum \limits_{k \neq k^*} \frac{b^2 \log T}{\Delta_k}\right) = O\left((K - k^*) \log T\right).
\end{equation*}
\end{theorem}

The proof is in Appendix~\ref{app:thm6}. 

\begin{remark}
    Algorithms~\ref{alg:corr_hoeffding} and \ref{alg:corr_bernstein} can be initialized by pulling arm $K$, so that every arm satisfies $m_j \geq 1$ and the initialization phase is complete in one cycle.\\
\end{remark}

\begin{remark}[Online variance computation]
Algorithms~\ref{alg:ind_bernstein} and~\ref{alg:corr_bernstein} compute $\hat\sigma_k^2$ from three running scalars $(n_k, S_k, Q_k)$ via $\hat\sigma_k^2 = (Q_k - S_k^2/n_k)/(n_k-1)$, where $S_k = \sum \limits_{i = 1}^{n_k} C_i$ and $Q_k = \sum \limits_{i = 1}^{n_k} C_i^2$. This requires $O(1)$ storage and $O(1)$ arithmetic per update.\\
\end{remark}

\begin{remark}\label{rem:LR_LB}
The lower bound (Section~\ref{sec:regret}) described in \citet{LaiRobbins} requires that any consistent algorithm pulls each suboptimal arm at least $\Omega(\log T)$ times.  The independent algorithms (Algorithms~\ref{alg:ind_hoeffding}--\ref{alg:ind_bernstein}) achieve $O(\log T)$ pulls for every suboptimal arm $k \neq k^*$, matching this lower bound up to constants.  The correlated algorithms improve upon this for arms $k < k^*$: these arms receive $O(1)$ direct pulls because each pull of any larger arm $j > k$ provides a free sample for arm $k$ (Proposition~\ref{prop:nested_obs}).
\end{remark}

\subsection{Survival Analysis based}\label{sec:Ren+KM}

Under a block replacement scheduling for machine maintenance, when a block replacement occurs at time $k$, every machine that has not yet failed only provides a censored observation of its lifetime because we do not observe its actual failure time. With traditional approaches that simply collect failure times and build an empirical distribution from complete observations alone, machines with longer lifetimes are disproportionately censored, leading to a biased sample that skews toward shorter lifetimes. This bias can result in suboptimal policies that replace machines too frequently. To address this, we develop an approach that properly accounts for censored observations using survival analysis techniques, specifically the Kaplan-Meier estimator, combined with renewal theoretic cost computation.

Consider the $t\textsuperscript{th}$ cycle with inter-block replacement interval $k_t = k$. Machine $j \in [N]$ generates $Y_t^j$ uncensored observations $(\ell, 1)$ corresponding to the complete lifetimes $L_1^j, \ldots, L_{Y_t^j}^j$, and one censored observation $(\ell, 0)$ for the residual time from its last failure to the block replacement epoch, since every machine is replaced at time $k$ regardless of its operational status. Aggregating over all $N$ machines, cycle $t$ contributes $\sum_{j=1}^N (Y_t^j + 1)$ observations. The accumulated dataset after $T$ cycles is $\mathcal{D}_T = \{(\ell_i, \delta_i)\}_{i=1}^{n_T}$, where \[n_T = \sum_{n=1}^{T}\sum_{j=1}^{N}(Y_n^j + 1) = NT + \sum_{n=1}^{T}\sum_{j=1}^{N}Y_n^j,\] with expected size $N(M(k)+1)$ per cycle of length $k$.

\subsubsection*{Kaplan-Meier (KM) Estimation}

Given the dataset $\mathcal{D}_T$ with censored observations, we estimate the lifetime distribution using the Kaplan-Meier (KM) estimator, which is the nonparametric maximum likelihood estimator for the survival function in the presence of right censoring. Let $t_1 < t_2 < \cdots < t_{n_T}$ denote the distinct observed times (both failures and censorings) in $\mathcal{D}_T$. At each time $t_j$, let $d_j$ denote the number of failures at time $t_j$ and let $n_j$ be the number of machines that have not yet failed just before time $t_j$. 

The \textit{Kaplan-Meier estimator} of the survival function is $\widehat{S}(t) = \prod_{j: t_j \leq t} \left(1 - \frac{d_j}{n_j}\right)$ for $t \geq 0$, with $\widehat{S}(0) = 1$. Set $\widehat{F}(t) = 1 - \widehat{S}(t)$ and $\widehat{f}(t) = \widehat{S}(t - 1) - \widehat{S}(t)$ and compute $\widehat{M}$ via the renewal equation (\ref{eqn:renewal}) setting \[\widehat{M}(t) = \widehat{F}(t) + \sum \limits_{s = 1}^{t} \widehat{f}(s) \widehat{M}(t - s)\] for $1 \leq t \leq K$ and $\widehat{M}(0) = 0$. Then estimate $\widehat{c}(k) = \frac{c_b + c_f \cdot N \cdot \widehat{M}(k)}{k}$.  The estimated optimal policy is $\widehat{k}^* = \argmin_{k \in [K]} \widehat{c}(k)$. 

Algorithm~\ref{alg:survival_renewal} has two hyperparameters: an exploration rate $\varepsilon$ and a refit frequency $R$.  A cold-start phase ($t \leq 5$) ensures initial KM coverage; thereafter the algorithm uses a decaying schedule $\varepsilon_t = \varepsilon/\sqrt{t}$, exploring uniformly with probability $\varepsilon_t$ and exploiting $\widehat{k}^*$ otherwise.  The parameter $R$ controls how often the KM estimator and renewal function are recomputed: smaller $R$ yields faster adaptation to accumulating data at higher per-cycle computational cost; larger $R$ amortizes computation but delays policy updates. 

We now establish the consistency and convergence properties of the survival-based renewal approach.\\

\begin{theorem}[\textit{Consistency of Kaplan-Meier Estimator}]\label{thm:km_consistency}
    Let $\widehat{F}_T$ denote the Kaplan-Meier estimator based on dataset $\mathcal{D}_T$ with $n_T$ observations. Suppose the censoring mechanism satisfies the independent censoring assumption. Then
\[
\sup_{t \in [K]} |\widehat{F}_T(t) - F(t)| \xrightarrow{a.s.} 0 \quad \text{as } n_T \to \infty.
\]
\end{theorem}

The proof is in Appendix~\ref{app:thm7}. 

\begin{theorem}[\textit{Policy and cost convergence}]\label{thm:kmConvergence}
    Let $\widehat{k}_T^*$ denote the estimated optimal policy based on dataset $\mathcal{D}_T$ with $n_T$ observations, and let $k^*$ denote the true optimal policy. Under the conditions of Theorem \ref{thm:km_consistency}, as $n_T \to \infty$, $\mathbb{P}(\widehat{k}_T^* = k^*) \to 1$. Furthermore, the estimated cost converges, that is $|\widehat{c}(\widehat{k}_T^*) - c(k^*)| \xrightarrow{a.s.} 0$. 
\end{theorem}

The proof is in Appendix~\ref{app:thm8}. 

\subsection{Computational Complexity Analysis}\label{sec:comp_comp}

The four bandit algorithms (Algorithms~\ref{alg:ind_hoeffding} -- \ref{alg:corr_bernstein}) share the same asymptotic computational complexity. Each maintains two or three running scalars per arm: $(n_k, S_k)$ for Hoeffding and $(n_k, S_k, Q_k)$ for Bernstein, requiring $O(K)$ space. Per decision epoch, computing all confidence bounds and selecting the minimum costs $O(K)$; the correlated update (accumulating statistics for all $j \leq k_t$) adds $O(K)$ to time complexity per epoch. Thus total time complexity is $O(KT)$ for all four algorithms.

The Kaplan-Meier algorithm maintains vectors $c(\cdot)$, $F(\cdot)$, and $M(\cdot)$ of length $K$, requiring $O(K)$ space. Updating $c(\cdot)$ and $F(\cdot)$ costs $O(K)$, but computing $\widehat{M}(t)$ via the renewal equation requires an $O(K^2)$ convolution. The per-epoch cost is therefore $O(K^2)$, giving $O(K^2 T)$ total. Despite the higher cost, the KM algorithm correctly incorporates all censored observations and in practice identifies $k^*$ with far fewer cycles than the bandit methods.

Value iteration for the age-vector MDP (Algorithm~\ref{alg:value_iteration_avg}) operates on the state space $\mathcal{S} = \{0,\ldots,K\}^N$ of cardinality $(K+1)^N$. Computing $Q_0^{(k)}(s)$ at each state requires summing over all $2^N$ failure subsets at cost $O(N)$ each, giving a per-iteration cost of $O\left(N\, 2^N (K+1)^N\right)$ and space $O\left((K+1)^N\right)$. The time-elapsed MDP (Algorithm~\ref{alg:vi-te}), by contrast, operates on a state space of cardinality $K + 1$, giving time complexity $O(K^2)$ and space $O(K)$ regardless of the fleet size $N$. 

Table~\ref{tab:complexity} summarizes the time and space complexity of all methods (here $I$ is the number of iterations taken to converge). Despite not having a formal regret guarantee comparable to those of the bandit methods (see Section~\ref{sec:algos}), the Kaplan-Meier approach can exhibit strong empirical performance due to its efficient use of censored data for explicit distribution learning.

\begin{table}[h]
\centering
\caption{Computational and space complexity of block replacement
algorithms.}
\label{tab:complexity}
\begin{tabular}{lcc}
\hline
Algorithm & Time & Space \\
\hline
All Four Bandit Algorithms   & $O(KT)$   & $O(K)$ \\
Kaplan-Meier Renewal        & $O(K^2T)$ & $O(K)$ \\ 
Time-Elapsed MDP & $O(K^2)$ & $O(K)$ \\
Age-vector MDP & $O(N 2^N (K+1)^N \cdot I)$ & $O((K+1)^N)$ \\
\hline
\end{tabular}
\end{table}

The choice of algorithm depends on the specific problem constraints: while bandit-based approaches offer significant computational advantages for large-scale online learning ($T \gg K$), the Kaplan-Meier method remains competitive for moderate horizons where its superior sample efficiency outweighs its relative computational cost.

\section{Numerical Experiments}\label{sec:numeric}

We evaluate all five data-driven algorithms (Algorithms \ref{alg:ind_hoeffding} -- \ref{alg:survival_renewal}) on two lifetime distributions under a common parameter regime and compare it against the age-vector MDP benchmark (Algorithm~\ref{alg:value_iteration_avg}). Throughout this section we fix $N = 2 \text{ machines}, K = 12, c_b = 1.0, c_f = 2.6$, and consider lifetimes drawn from
\begin{enumerate}[label=(\roman*)]
  \item $L \sim 1 + \mathrm{Binomial}(10,\, 0.5)$, with support $\{1,\ldots,11\}$
    and mean $\mu = 6$;
  \item $L \sim 1 + \mathrm{Poisson}(4)$, with support $\{1, 2, \ldots\}$ and
    mean $\mu = 5$.
\end{enumerate}
Note that $c_b = 1 < 2c_f = 5.2$, so block replacement is strictly cheaper per machine than two simultaneous individual replacements, satisfying the assumption in Section~\ref{sec:setup}. The binomial distribution has bounded support, while the Poisson distribution has unbounded support; for the latter, the true renewal function $M(k)$ and cost curve $c(k)$ are computed using a PMF truncated at $K_{\mathrm{renew}} = 45$, which captures at least $99.9\%$ of the probability mass. 

Since $c_f$ and $N$ enter the cost function only through their product $c_f N$, our two-machine setup (which is computationally simple) with $c_f = 2.6$ is equivalent to an $N$-machine fleet with per-failure cost $5.2/N$, so the results scale directly to larger fleets. Our results therefore apply directly to larger fleets operating under a proportionally smaller per-failure cost. As long as $c_b$ and $c_f \cdot N$ are fixed, the true underlying cost structure will not change, but the MDP formulation will, as its state space grows as $(K + 1)^N$. 

\subsection{Age-vector MDP Formulation (Algorithm~\ref{alg:value_iteration_avg})}\label{sec:MDP_simul}

We implement Algorithm~\ref{alg:value_iteration_avg} for the two-machine MDP ($N = 2$) to illustrate the structural properties established in Theorem~\ref{thm:mdp_structure} and to compute the benchmark optimal average cost $g^*_{\mathrm{age}}$ against which the bandit and Kaplan--Meier algorithms are compared. The state space $\mathcal{S} = \{0,1,\ldots,12\}^2$ (with $|\mathcal{S}| = 169$ states), where each coordinate records the current age of a machine capped at $K = 12$. Value iteration is run with span-seminorm convergence tolerance $\varepsilon = 10^{-8}$. The resulting optimal policy is summarised in Table~\ref{tab:mdp_policy_combined} via its threshold function (Theorem~\ref{thm:mdp_structure}(d)).

\begin{table}[ht]
\centering
\caption{Threshold function $\tau^*(b)$ for the optimal MDP policy under Algorithm~\ref{alg:value_iteration_avg}. By Theorem~\ref{thm:mdp_structure}(c), the policy is symmetric, so $\tau^*_1(b) = \tau^*_2(a) =: \tau^*(b)$: block replacement is triggered whenever either machine's age reaches or exceeds $\tau^*(b)$, where $b$ is the other machine's current age. A threshold of 0 indicates that block replacement is triggered immediately regardless of the focal machine's age.}
\label{tab:mdp_policy_combined}
\small
\begin{tabular}{l *{13}{c}}
\toprule
Other machine's age $b$ & 0 & 1 & 2 & 3 & 4 & 5 & 6 & 7 & 8 & 9 & 10 & 11 & 12 \\
\midrule
$L \sim 1+\mathrm{Binomial}(10,0.5)$ & 4 & 4 & 3 & 2 & 0 & 0 & 0 & 0 & 0 & 0 & 0 & 0 & 0 \\
$L \sim 1+\mathrm{Poisson}(4)$       & 3 & 2 & 1 & 0 & 0 & 0 & 0 & 0 & 0 & 0 & 0 & 0 & 0 \\
\bottomrule
\end{tabular}
\end{table}

The optimal average costs are $g^*_{\mathrm{age}} \approx 0.3195$ for Binomial and $g^*_{\mathrm{age}} \approx 0.4886$ for Poisson. By Theorem~\ref{thm:mdp_structure}(c), the policy is symmetric in the two machine ages, so a single threshold function suffices: block replacement is triggered for machine $i$ as soon as its age reaches $\tau^*(b)$, where $b$ is the current age of the other machines. The threshold is non-increasing in $b$ (policy monotonicity), the function is identical for both machines (symmetry), and $\tau^*(b) = 0$ once the other machine is old enough (threshold structure collapses to immediate replacement).

The values of $\tau^*(b)$ for both the Binomial and Poisson case are presented in Table~\ref{tab:mdp_policy_combined}. Note that $\tau^* = 0$, meaning any age triggers block replacement. In both cases the threshold drops sharply to zero once the companion machine reaches a moderate age, and the high ratio $c_f N / c_b = 5.2$ makes proactive replacement attractive even for relatively young machines.

\subsection{Bandit and Survival Analysis based methods (Algorithms \ref{alg:ind_hoeffding} - \ref{alg:survival_renewal})}\label{sec:BanditKM_Simul}

We evaluate the four bandit algorithms (Algorithms~\ref{alg:ind_hoeffding}--\ref{alg:corr_bernstein}) and the Kaplan--Meier renewal algorithm (Algorithm~\ref{alg:survival_renewal}) under the same cost parameters and lifetime distributions as described above. Each algorithm operates over $T = 10000$ decision epochs, using the same sequence of underlying random failure times for all algorithms for a fair comparison. 

The true cost curves $c(k)$ for $k \in [K]$, shown in Figure~\ref{fig:trueCosts}, reveal the following optimal block replacement intervals:
\begin{enumerate}[label=(\roman*)]
  \item $k^* = 3$ for $L \sim 1 + \mathrm{Binomial}(10, 0.5)$,
    with $c(k^*) \approx 0.4282$;
  \item $k^* = 2$ for $L \sim 1 + \mathrm{Poisson}(4)$,
    with $c(k^*) \approx 0.7390$.
\end{enumerate}

\begin{figure}[h]
    \centering
    \includegraphics[width=1\linewidth]{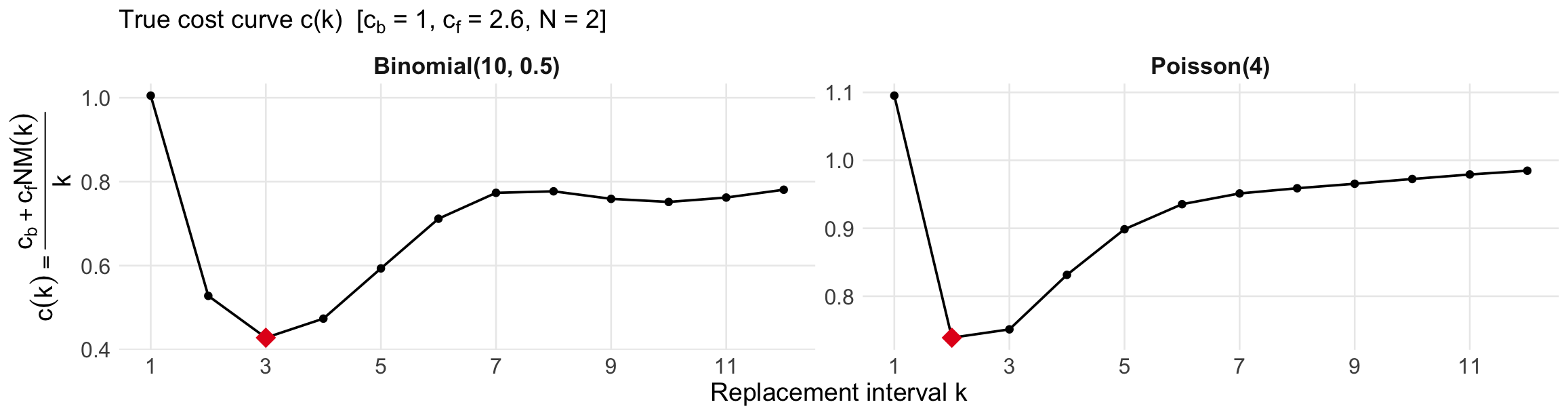}
    \caption{True long-run average cost $c(k)$ vs $k$ under each lifetime distribution. The optimal interval $k^*$ is marked with a red diamond.}
    \label{fig:trueCosts}
\end{figure}

For each distribution the true optimal interval $k^*$ and cost $c(k^*)$ are computed via the renewal equation~\eqref{eqn:renewal} using a PMF truncated at $K_{\mathrm{renew}} = 45$ to ensure accuracy under the Poisson tail. 

We have tested all five algorithms using simulations. The simulations need to run over long horizons for the long-run regret asymptotics to manifest itself. All five algorithms use a common random seed. The results are from a single run.

\begin{figure}[h]
    \centering
    \includegraphics[width=1\linewidth]{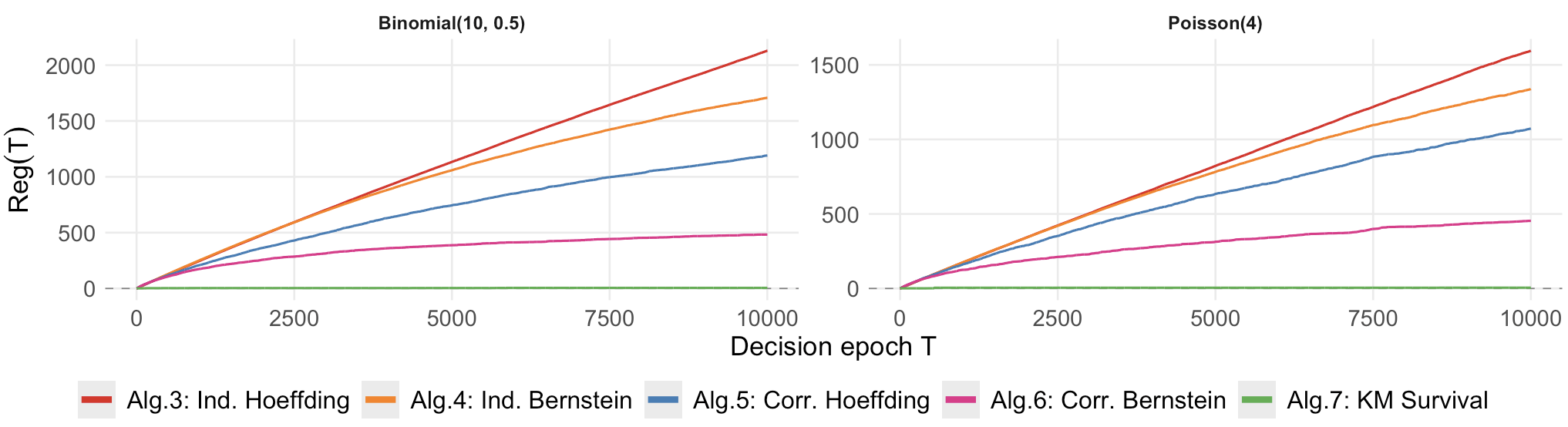}
    \caption{Cumulative regret Reg($T$) for Algorithms \ref{alg:ind_hoeffding}--\ref{alg:survival_renewal} over $T = 10,000$ decision epochs ($c_b = 1$, $c_f = 2.6$, $N = 2$, $K = 12$).}
    \label{regrets}    
    \label{fig:regret}
\end{figure}

\begin{figure}[h]
    \centering
    \includegraphics[width=1\linewidth]{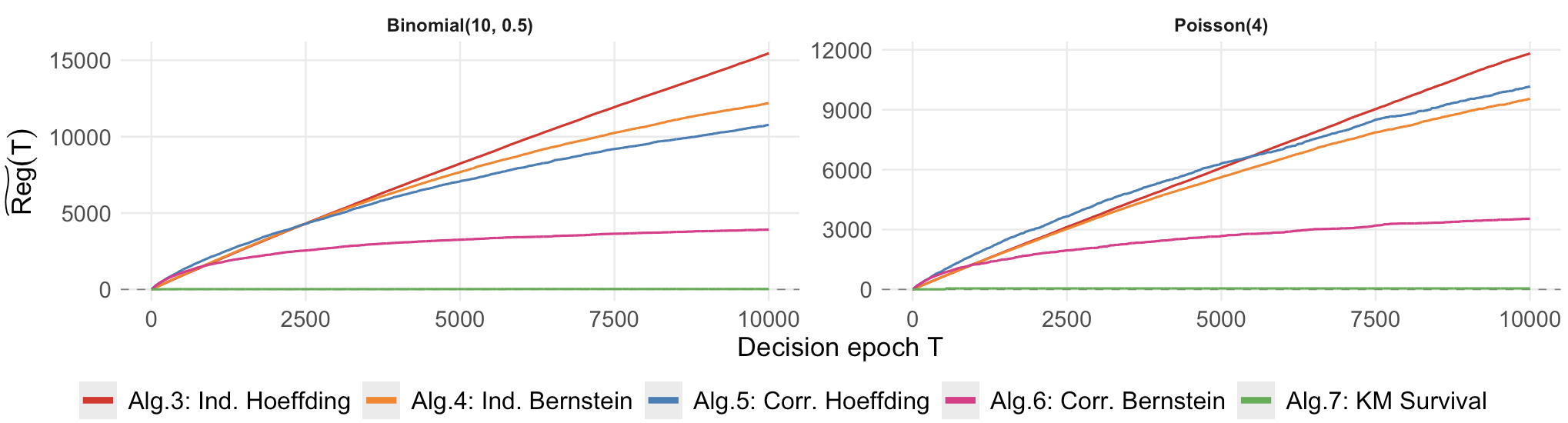}
    \caption{Cumulative regret $\widetilde{\text{Reg}}(T)$ for Algorithms \ref{alg:ind_hoeffding}--\ref{alg:survival_renewal} over $T = 10,000$ decision epochs ($c_b = 1$, $c_f = 2.6$, $N = 2$, $K = 12$).}
    \label{fig:regretTilde}
\end{figure}

Figure \ref{fig:regret} reports the \emph{per-cycle} cumulative regret (Equation~\ref{eqn:regDecomp}) and Figure \ref{fig:regretTilde} the \emph{total-cost} cumulative regret (Equation~\ref{eqn:regTilDecomp}) over $T = 10,000$ epochs across the five algorithms and two distributions. All four bandit algorithms exhibit sub-linear regret growth, consistent with the established upper bounds for Algorithms~\ref{alg:ind_hoeffding}--\ref{alg:corr_bernstein}. Algorithms \ref{alg:corr_hoeffding} and \ref{alg:corr_bernstein} (correlated arm bandits) accumulate substantially less regret than Algorithms \ref{alg:ind_hoeffding} and \ref{alg:ind_bernstein} (independent arm bandits), confirming the theoretical $O((K - k^*)\log T)$ versus $O((K-1)\log T)$ separation. Algorithms \ref{alg:corr_hoeffding} and \ref{alg:corr_bernstein} exploit empirical variance information, yielding tighter confidence bounds and faster arm elimination when the reward variance differs across arms. 

The cumulative regret plots reveal a clear performance ordering: among the bandit algorithms, the correlated Bernstein LCB (Algorithm~\ref{alg:corr_bernstein}) achieves the lowest regret, consistent with its $O((K-k^*)\log T)$ guarantee and tighter variance-aware confidence bounds; the Kaplan--Meier algorithm (Algorithm~\ref{alg:survival_renewal}) outperforms all bandit methods overall. By explicitly learning $F$ from censored observations and computing $c(k)$ via the renewal equation, the KM algorithm identifies $k^*$ rapidly after the cold-start phase and thereafter incurs near-zero incremental regret. Since operating a system for $T = 10{,}000$ cycles to achieve asymptotic performance is rarely feasible in practice, we examine arm-selection behaviour of these two best-performing algorithms over the first $T = 50$ decision epochs; Figure~\ref{fig:decisions} shows the frequency of each chosen interval $k$ over this short horizon.

\begin{figure}[h]
    \centering
    \includegraphics[width=1\linewidth]{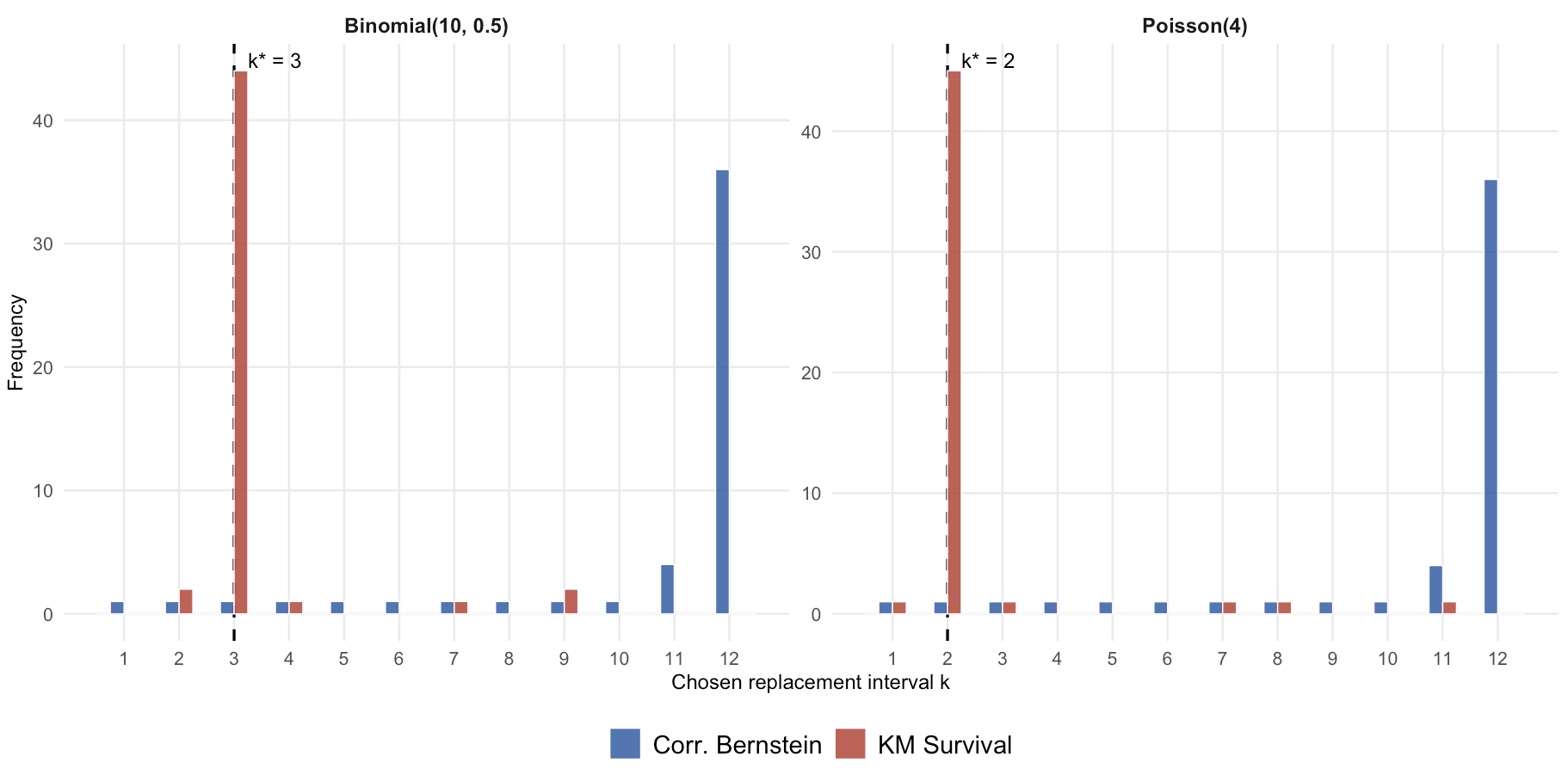}
    \caption{Histogram of replacement intervals $k$ selected by the Correlated Bernstein LCB (Algorithm~\ref{alg:corr_bernstein}) and the Kaplan--Meier Renewal (Algorithm~\ref{alg:survival_renewal}) algorithms over the first $T = 50$ decision epochs, under the Binomial$(10,0.5)$ (left) and Poisson$(4)$ (right) lifetime distributions. The dashed vertical line marks the optimal interval $k^*$.}
    \label{fig:decisions}
\end{figure}

The concentration of the correlated Bernstein algorithm on large arms reflects necessary exploration: the $O((K-k^*)\log T)$ regret guarantee requires every suboptimal arm above $k^*$ to be pulled $O(\log T)$ times to keep confidence bounds calibrated, and at $T = 50$ this amounts to visiting each such arm roughly $\log(50) \approx 3.9$ times. This cost is bounded and vanishes as a fraction of total cost as $T \to \infty$. The KM algorithm closes in on the optimal arm surprisingly quickly, using the optimal arm most of the time. We find this insight most revealing and useful in practice. 

\subsection{Comparison of Data-Driven algorithms against MDP benchmark}\label{sec:MDPvsDD_Comp}

The MDP formulation in Section~\ref{sec:MDP} optimises over all age-dependent Markovian policies using complete knowledge of the lifetime distribution $F$, providing a lower bound on the long-run average cost. The data-driven algorithms of Section~\ref{sec:bandit} and \ref{sec:Ren+KM} operate without knowledge of $F$ and are restricted to block replacement policies. The gap between their empirical costs and $g^*_{\mathrm{age}}$ can be decomposed as $c(\widehat{k}^*(T)) - g^*_{\mathrm{age}} = \underbrace{c(k^*) - g^*_{\mathrm{age}}}_{\delta_{\mathrm{struct}}} + \underbrace{c(\widehat{k}^*(T)) - c(k^*)}_{\delta_{\mathrm{learn}}(T)}$, where $\widehat{k}^*(T) := \mathrm{mode}_{t \in [T-W, T]} k_t$ is the arm most frequently selected by the algorithm in the terminal window $[T-W, T]$ with $W = 1000$ cycles, $\delta_{\mathrm{struct}}$ is the \emph{structural gap} incurred by restricting to block policies, and $\delta_{\mathrm{learn}}(T)$ is the \emph{learning gap} due to not knowing $F$. The intermediate value $c(k^*)$ is the optimal cost of the time-elapsed MDP (Section~\ref{sec:time-elapsed-mdp}), computed analytically in $O(K^2)$ operations via Algorithm~\ref{alg:vi-te}; no simulation of the time-elapsed MDP is required.

All five algorithms identify $\widehat{k}^*(T) = k^*$ by the end of the horizon under both distributions, yielding $\delta_{\mathrm{learn}}(T) = 0\%$ in every case. The costs at termination are $c(\widehat{k}^*) = c(k^*) \approx 0.4646$ (Binomial) and $c(\widehat{k}^*) = c(k^*) \approx 0.8309$ (Poisson), so the entire suboptimality relative to $g^*_{\mathrm{age}}$ is structural. Specifically, $\delta_{\mathrm{struct}} \approx 0.1166 \ (33.5\%\ \text{of}\ g^*_{\mathrm{age}})$ for Binomial and $\delta_{\mathrm{struct}} \approx 0.2780 \ (50.3\%\ \text{of}\ g^*_{\mathrm{age}})$ for Poisson. Even an oracle with full knowledge of $F$ that always selects the best fixed-interval block policy pays a cost above the best age-dependent policy.

\section{Conclusion}
\label{sec:conclusion}

We have developed and analysed a family of data-driven algorithms for block replacement scheduling of $N$ identical machines with unknown lifetime distribution $F$. The independent and correlated LCB bandit algorithms (Algorithms~\ref{alg:ind_hoeffding}--\ref{alg:corr_bernstein}) achieve $O(\log T)$ regret, matching the \citet{LaiRobbins} lower bound. The correlated variants require only $O(1)$ direct pulls of arms $k < k^*$ by exploiting the nested observation structure unique to block replacement. The Kaplan--Meier renewal algorithm (Algorithm~\ref{alg:survival_renewal}) explicitly estimates $F$ from the censored operational record and identifies $k^*$ with near-zero incremental regret after a short exploration phase, empirically dominating all bandit algorithms at long horizons. At short horizons of $T = 50$ decision epochs, the survival analysis based algorithm already concentrates its selections near $k^*$, while the correlated Bernstein algorithm spreads pulls across larger arms due to its mandatory exploration phase, a transient cost that vanishes as $T \to \infty$. The underlying average-cost MDP admits a monotone threshold policy under IFR lifetime distributions (Theorem~\ref{thm:mdp_structure}), providing a gold-standard benchmark $g^*_{\mathrm{age}}$.

The framework assumes stationary IID machine lifetimes; real systems may exhibit load-dependent degradation or temporal trends. Establishing a regret bound for Algorithm~\ref{alg:survival_renewal} and developing a Thompson sampling variant that places a prior on $F$ and updates via censored observations are natural extensions of the proposed approaches. Bayesian approaches, including Thompson sampling (\citet{Thompson1933, Russo2018}) and related posterior sampling methods, which places a prior on the arm reward distributions and selects arms by posterior sampling, offer competitive empirical performance in standard bandit settings and represent a natural alternative to the LCB framework developed here.

Extending the framework to lifetimes that depend on environmental covariates or accumulated degradation, for example by incorporating a Cox proportional hazards model into the KM step, or by adopting contextual bandit algorithms, would substantially broaden the practical applicability of the proposed framework.

The structural gap $\delta_{\mathrm{struct}} = c(k^*) - g^*_{\mathrm{age}}$ arises because fixed-interval block replacement discards the age information available at each decision point.  A natural intermediate policy is the $(k^*, \ell^*)$ \emph{hybrid policy}: block-replace only machines with ages at least $\ell^*$ every $k^*$ periods. A data-driven algorithm must learn both $k^*$ and $\ell^*$ from censored observations. The reduction in $\delta_{\mathrm{struct}}$ achievable by the optimal $\ell^*$ quantifies the value of incorporating simple age-dependent information into an otherwise schedule-driven policy.

The present framework assumes all $N$ machines are statistically identical.  When machines have different age profiles or lifetime distributions, as arises in mixed-age fleets following partial replacements, the bandit formulation must be extended to account for the differing contribution of each machine to the block replacement cost, a setting with closer connections to the restless bandit literature.

% NOTE: Use the Code and Data Disclosure section to provide instructions for where your code and data, along with README file, can be found. If the paper received an exemption, then state the reason the exemption was granted.

\newpage

\appendix

\section*{Appendix 1: Proofs of Theorems and Propositions}

    \section{Proof of Theorem~\ref{thm:mdp_structure}}  \label{app:thm2}
\begin{enumerate}[label = (\alph*)]
    \item  We prove by induction that the iterates $\{w^{(n)}\}$ satisfy two properties simultaneously:
\begin{itemize}
  \item[(A)] \emph{Monotonicity:} For every $s \in \mathcal{S}$ and every $i \in [N]$ with $s^i \leq K-1$, $w^{(n)}(s + e_i) \geq w^{(n)}(s)$.
  \item[(B)] \emph{Lipschitz bound:} For every $r \in \mathcal{S}$ and $i \in [N]$, $w^{(n)}(r) - w^{(n)}(r^{\{i\}}) \leq c_f$, where $r^{\{i\}}$ denotes $r$ with coordinate $i$ reset to $0$.
\end{itemize}
The restriction $s^i \leq K-1$ in (A) is necessary to ensure $s + e_i \in \mathcal{S}$. 

\emph{Base case.} Both properties hold trivially for $w^{(0)} \equiv 0$.

\emph{Inductive step.} Assume (A) and (B) hold for $w^{(n)}$ for some $n \in \N$. Fix $i \in [N]$. For brevity, write $V^+_{A'} := c_f + w^{(n)}(s^{[A' \cup \{i\}]})$ for the value following machine $i$'s failure, and $V^0_{A'} := w^{(n)}(s^{[A']})$ for the value following its survival. Define \[\Phi_{A'}(s) \;:=\; h(s^i{+}1)\,V^+_{A'} \;+\; \bar{h}(s^i{+}1)\,V^0_{A'}.\]  Conditioning on which subset $A'$ of the other $N-1$ machines fail,
\begin{align}\label{eq:Q0split}
  Q_0^{(n)}(s)
    &= c_f \sum_{j \neq i} h(s^j{+}1) \\
    &\quad + \sum_{A' \subseteq [N]\setminus\{i\}}
             \omega_{A'}(s)\,\Phi_{A'}(s). \nonumber
\end{align}
where $\omega_{A'}(s) := \prod_{j\in A'} h(s^j{+}1)\prod_{j\notin A',\, j\neq i} \bar{h}(s^j{+}1) \geq 0$ with $\sum \limits_{A'}\omega_{A'}(s) = 1$. In $s^{[A']}$, machine $i$ survives and its age advances to $s^i + 1$; in $s^{[A'\cup\{i\}]}$, machine $i$ fails and resets to age $0$.

\medskip
\emph{Proof of (A) for $w^{(n+1)}$.}
Fix $s \in \mathcal{S}$ with $s^i \leq K-1$, so that $s + e_i \in \mathcal{S}$. 

Replace the lead-in and equation (14) with:

Let $\tilde{s}^{[A']}$ denote the next state from $s+e_i$ when machines in $A'$ fail and machine $i$ survives: $\tilde{s}^{[A']}$ has coordinate $i$ equal to $s^i+2$ (with zero weight when $s^i+1=K$), coordinate $j=0$ for $j \in A'$, and coordinate $j = s^j+1$ for $j \notin A' \cup \{i\}$. Note $\tilde{s}^{[A']} \succeq s^{[A']}$ coordinatewise and $(\tilde{s}^{[A']})^{\{i\}} = s^{[A'\cup\{i\}]}$. Note that $\delta^i_h := h(s^i{+}2)-h(s^i{+}1) \geq 0$ by the IFR assumption. Define $R_1(A') := c_f + w^{(n)}(s^{[A'\cup\{i\}]}) - w^{(n)}(\tilde{s}^{[A']})$ and $R_2(A') := w^{(n)}(\tilde{s}^{[A']}) - w^{(n)}(s^{[A']})$. 
Then
\begin{align}
  &Q^{(n)}_0(s+e_i) - Q^{(n)}_0(s) \notag \\
    &= \sum_{A'} \omega_{A'}(s)\left[\delta^i_h \cdot R_1(A') + \bar{h}(s^i{+}1)\cdot R_2(A')\right]. \label{eq:Qdiff}
\end{align}
The term $R_2(A') \geq 0$ by inductive hypothesis~(A), since $\tilde{s}^{[A']} \succeq s^{[A']}$. For $R_1(A')$, applying~(B) to $\tilde{s}^{[A']}$ at coordinate $i$ gives $w^{(n)}(\tilde{s}^{[A']}) - w^{(n)}((\tilde{s}^{[A']})^{\{i\}})  = w^{(n)}(\tilde{s}^{[A']}) - w^{(n)}(s^{[A'\cup\{i\}]}) \leq c_f$, so $R_1(A') \geq 0$. Since $\delta^i_h \geq 0$, both terms are non-negative, giving $Q^{(n)}_0(s+e_i) \geq Q^{(n)}_0(s)$.

\medskip
\emph{Proof of (B) for $w^{(n+1)}$.}
Fix $r \in \mathcal{S}$ and $i \in [N]$. Since $Q_0^{(n)}(r) \geq Q_0^{(n)}(r^{\{i\}})$ by (A) (as $r^{\{i\}} + e_i$ has coordinate $i$ equal to $1 \leq K$), and $\min(a, c) - \min(b, c) \leq a - b$ for any reals $a \geq b$:
\[w^{(n+1)}(r) - w^{(n+1)}(r^{\{i\}}) \leq Q_0^{(n)}(r) - Q_0^{(n)}(r^{\{i\}}).\]
We apply~\eqref{eq:Q0split} to both $r$ and $r^{\{i\}}$ and subtract. When machine $i$ fails, both $r$ and $r^{\{i\}}$ produce next-state coordinate $i = 0$, so $r^{[A'\cup\{i\}]} = (r^{\{i\}})^{[A'\cup\{i\}]}$. When machine $i$ survives from $r$ it ages to $r^i+1$; from $r^{\{i\}}$ it ages to $1$. Define $B_1(A') := c_f + w^{(n)}(r^{[A'\cup\{i\}]}) - w^{(n)}((r^{\{i\}})^{[A']})$ and $B_2(A') := w^{(n)}(r^{[A']}) - w^{(n)}((r^{\{i\}})^{[A']})$.

By inductive hypothesis~(A), $r^{[A']} \succeq (r^{\{i\}})^{[A']}$ at coordinate $i$ and $(r^{\{i\}})^{[A']} \succeq r^{[A'\cup\{i\}]}$ at coordinate $i$, so both $B_1(A') \geq 0$ and $B_2(A') \geq 0$. Applying~(B) to $(r^{\{i\}})^{[A']}$ at coordinate $i$ gives $B_1(A') \leq c_f$, and applying~(B) to $r^{[A']}$ at coordinate $i$ gives $B_2(A') \leq c_f$. Since $h(r^i+1) - h(1) \geq 0$ by IFR,
\begin{align}
  &Q^{(n)}_0(r) - Q^{(n)}_0(r^{\{i\}}) \notag\\
    &= \sum_{A'} \omega_{A'}(r)\left[
        \left(h(r^i{+}1)-h(1)\right) B_1(A') \right. \notag\\
    &\hspace{3.5em}\left.
      +\bar{h}(r^i{+}1)\, B_2(A')
      \right].
\end{align}
Using $\sum_{A'} \omega_{A'} = 1$, $\bar{h}(r^i+1) \leq 1$, and the bounds $B_1, B_2 \leq c_f$, we obtain \[Q^{(n)}_0(r) - Q^{(n)}_0(r^{\{i\}}) \leq c_f\bigl(h(r^i{+}1)-h(1)+1-h(r^i{+}1)\bigr) = c_f(1-h(1)) \leq c_f,\] establishing~(B) for $n+1$.

\medskip
\emph{Transfer to $w^*$.} Under the unichain condition (Theorem~\ref{thm:value_iter_convergence}), $\mathrm{sp}(w^{(n+1)} - w^{(n)}) \to 0$. Since $w^{(n)}(\mathbf{0}) = 0$ for all $n$, the sequence $\{w^{(n)}\}$ is uniformly bounded and converges pointwise to $w^*$. Properties (A) and (B) are preserved under pointwise limits, so $w^*$ satisfies:
\begin{itemize}
  \item[(A)] $w^*(s+e_i) \geq w^*(s)$ for all $s$ with $s^i \leq K-1$, i.e., $w^*$ is non-decreasing in each coordinate.
  \item[(B)] $w^*(r) - w^*(r^{\{i\}}) \leq c_f$ for all $r \in \mathcal{S}$, $i \in [N]$.
\end{itemize}
This completes the proof of part~(a).

\item Define the advantage function  $\Delta({s}) := Q_0^*({s}) - Q_1^*({s})$, where $Q_1^*({s}) = c_b +  h^*({0})$ is constant in ${s}$. The optimal policy satisfies $\pi^*({s}) = 1$  if and only if $\Delta({s}) \geq 0$. Since $Q_0^*({s}+{e}_i) \geq Q_0^*({s})$ for all $i \in [N]$, we have $Q_0^*({s})$ is non-decreasing in each  coordinate, and hence $\Delta({s})$ is non-decreasing in each coordinate. If $\pi^*({s}) = 1$ then  $\Delta({s}) \geq 0$, and for any ${s}' \geq {s}$,  $\Delta({s}') \geq \Delta({s}) \geq 0$, so $\pi^*({s}') = 1$.

\item By symmetry of the problem setup, permuting machines $1, 2, \cdots, N$ yields an equivalent MDP. Hence both the value function and optimal policy inherit this symmetry.

\item Fix $i \in [N]$ and ${s}_{-i} \in \{0, 1, \cdots, K\}^{N - 1}$. Define $\tau^*_i({s}_{-i}) := \min\{s^i \in \{0,\ldots,K\} : \pi^*(s^i, {s}_{-i}) = 1\}$, setting $\tau^*_i({s}_{-i}) = \infty$ if no such $s^i$ exists. If $s^i \geq \tau^*_i({s}_{-i})$, then $(s^i, {s}_{-i}) \geq (\tau^*_i({s}_{-i}), {s}_{-i})$ coordinatewise and $\pi^*(\tau^*_i({s}_{-i}), {s}_{-i}) = 1$ by definition, so Theorem~\ref{thm:mdp_structure}(b) gives $\pi^*(s^i, {s}_{-i}) = 1$. Conversely, if $s^i < \tau^*_i({s}_{-i})$, then $\pi^*(s^i, {s}_{-i}) = 0$ by minimality of $\tau^*_i({s}_{-i})$. This establishes the threshold representation.
    
For monotonicity, suppose ${s}_{-i}' \geq {s}_{-i}$ componentwise and $\tau^*_i({s}_{-i}) < \infty$. Setting $s^i_0 := \tau^*_i({s}_{-i})$, we have $\pi^*(s^i_0, {s}_{-i}) = 1$ and $(s^i_0, {s}_{-i}') \geq (s^i_0, {s}_{-i})$ coordinatewise, so Theorem~\ref{thm:mdp_structure}(b) gives $\pi^*(s^i_0, {s}_{-i}') = 1$, hence $\tau^*_i({s}_{-i}') \leq s^i_0 = \tau^*_i({s}_{-i})$. \hfill $\blacksquare$
    \end{enumerate}

\section{Proof of Theorem~\ref{thm:corr_hoeff}} \label{app:thm5}
Using the regret decomposition in Equation~\ref{eqn:regDecomp}, it suffices to show that $\E[n_k(T)] = O(1)$ for all $k < k^*$. Fix $k < k^*$. Since $k^*$ is assumed to be unique, $\Delta_k = c(k) - c(k^*) > 0$. Define the critical sample size $n_0 = \lceil 16b^2\Delta_k^{-2}\log(KT^2)\rceil$ where $b = c_f N$. By Hoeffding's inequality, for any $j \in [K]$ with $m_j(t) \geq n_0$ and $m_{k^*}(t) \geq n_0$, \[\P\left(\left|\widehat{c}_j^{(m)} - c(j)\right| \geq \tfrac{\Delta_k}{4}\right) \leq 2\exp\left(-\tfrac{m\Delta_k^2}{8b^2}\right) \leq \frac{2}{KT^2}.\]  Define the event $\mathcal{G}_t = \left\{\widehat{c}_k^{(m)}(t) \leq c(k) + \Delta_k/4\right\} \cap \left\{\widehat{c}_{k^*}^{(m)}(t) \geq c(k^*) - \Delta_k/4\right\}$. By union bound, $\mathbb{P}(\mathcal{G}_t^c) \leq 4/(KT^2)$. For $m \geq n_0$, \[b\sqrt{\tfrac{2\log t}{m}} \leq b\sqrt{\tfrac{2\log(KT^2)}{n_0}} = b\sqrt{\tfrac{2\log(KT^2)}{16b^2\Delta_k^{-2}\log(KT^2)}} = \tfrac{\Delta_k}{4}.\] 
On $\mathcal{G}_t$, the LCB values therefore satisfy (i) $\mathrm{LCB}_k^{CH}(t) \geq c(k) - \tfrac{\Delta_k}{4} - \tfrac{\Delta_k}{4} = c(k) - \tfrac{\Delta_k}{2}$ and also (ii) $\mathrm{LCB}_{k^*}^{CH}(t) \leq c(k^*) + \tfrac{\Delta_k}{4} + \tfrac{\Delta_k}{4} = c(k^*) + \tfrac{\Delta_k}{2}$. Since $c(k)-c(k^*)=\Delta_k$, by (i) and (ii), we have \[\mathrm{LCB}_k^{CH}(t) \geq c(k^*)+\tfrac{\Delta_k}{2} > \mathrm{LCB}_{k^*}^{CH}(t),\] so arm $k$ is not selected at epoch $t$. Arm $k$ can only be pulled when $m_k(t)<n_0$, or $m_{k^*}(t)<n_0$, or $\mathcal{G}_t^c$ occurs. Thus, \[\E[n_k(T)] \leq 2n_0 + \sum \limits_{t=1}^T \P(\mathcal{G}_t^c) \leq 2n_0 + \frac{4}{K}\sum \limits_{t=1}^\infty \frac{1}{t^2} = 2n_0 + \frac{2\pi^2}{3K} = O(1).\]

\section{Proof of Theorem~\ref{thm:corr_bern}}\label{app:thm6}     As in the proof of the previous theorem, it suffices to show that $\E[n_k(T)] = O(1)$ for all $k < k^*$. Fix $k < k^*$. Since $k^*$ is unique, $\Delta_k = c(k) - c(k^*) > 0$. Let $\sigma_k^2 = \text{Var}\left(\frac{C^{(k)}}{k}\right)$ and $\sigma_{k^*}^2 = \text{Var}\left(\frac{C^{(k^*)}}{k^*}\right)$ denote the variances, with $\sigma_k^2, \sigma_{k^*}^2 \leq b^2$ where $b = c_f N$.

Define $n_0 = \left\lceil\max\left\{ \frac{96\sigma_k^2}{\Delta_k^2}\log(KT^2), \frac{48b}{\Delta_k}\log(KT^2) \right\}\right\rceil$ and $n_{0,k^*} = \left\lceil\max\left\{ \frac{96\sigma_{k^*}^2}{\Delta_k^2}\log(KT^2),\frac{48b}{\Delta_k}\log(KT^2) \right\}\right\rceil$ to be the critical sample sizes. By Bernstein's inequality, for $m \geq n_0$ and any arm $j$, \[\P\left(\left|\widehat{c}_j^{(m)}-c(j)\right|\geq\frac{\Delta_k}{4}\right) \leq 2\exp \left(-\frac{m\Delta_k^2/16}{2\sigma_j^2 + b\Delta_k/6}\right) \leq \frac{2}{KT^2}.\] Moreover, for $m \geq n_0$, by Bernstein's inequality, $\P\left((\widehat{\sigma}_j^{(m)})^2 \geq 2\sigma_j^2\right) \leq \frac{2}{KT^2}$. 

Define \[\mathcal{G}_t := \Bigl\{|\widehat{c}_k^{(m)}(t)-c(k)|\leq\tfrac{\Delta_k}{4}\Bigr\} \cap \Bigl\{|\widehat{c}_{k^*}^{(m)}(t)-c(k^*)|\leq\tfrac{\Delta_k}{4}\Bigr\} \cap \Bigl\{(\hat\sigma_k^{(m)})^2(t)\leq 2\sigma_k^2\Bigr\} \cap \Bigl\{(\hat\sigma_{k^*}^{(m)})^2(t)\leq 2\sigma_{k^*}^2\Bigr\}.\] When $m_k(t)\geq n_0$ and $m_{k^*}(t)\geq n_{0,k^*}$, the union bound gives $\P(\mathcal{G}_t^c)\leq 8/(KT^2)$. Since $(\hat\sigma_{k^*}^{(m)})^2(t) \leq 2\sigma_{k^*}^2$ on $\mathcal{G}_t$ and $m_{k^*}(t)\geq n_{0,k^*}\geq 2$, we have \[\mathrm{LCB}_{k^*}^{CB}(t) \leq c(k^*) + \frac{\Delta_k}{4} - \sqrt{\frac{2.4\cdot 2\sigma_{k^*}^2\log t}{n_{0,k^*}}} - \frac{3.6b\log t}{n_{0,k^*}}.\] By definition of $n_{0,k^*}$, $\sqrt{4.8\sigma_{k^*}^2\log t / n_{0,k^*}}\leq\sqrt{4.8\sigma_{k^*}^2/96\sigma_{k^*}^2} \cdot\Delta_k = \Delta_k/\sqrt{20}$, and $3.6b\log t/n_{0,k^*}\leq 3.6b/(48b)\cdot\Delta_k = \Delta_k/\sqrt{50}$ (both uniformly in $t\leq T$). Hence \[\mathrm{LCB}_{k^*}^{CB}(t) \leq c(k^*) + \frac{\Delta_k}{4} - \frac{\Delta_k}{\sqrt{20}} - \frac{\Delta_k}{\sqrt{50}}  \leq c(k^*) + \frac{\Delta_k}{4} - \frac{\Delta_k}{5} < c(k^*) + \frac{\Delta_k}{8}.\]

Since $\left(\hat\sigma_k^{(m)}\right)^2(t)\leq 2\sigma_k^2\leq b^2$ on $\mathcal{G}_t$ and $\widehat{c}_k^{(m)}(t)\geq c(k)-\Delta_k/4$, we have \[\mathrm{LCB}_k^{CB}(t) \geq c(k) - \frac{\Delta_k}{4} - \sqrt{\frac{2.4\cdot 2\sigma_k^2\log t}{m_k(t)}} - \frac{3.6b\log t}{m_k(t)}.\] Even if $m_k(t) = 1$ (Hoeffding bound applies), $\mathrm{LCB}_k^{CB}(t) \geq c(k) - \Delta_k/4 - b\sqrt{2\log t}$, and arm $k$ is not selected once $n_0$ is reached regardless.  For $m_k(t) \geq n_0$, a similar argument gives \[\mathrm{LCB}_k^{CB}(t) \geq c(k) - \frac{\Delta_k}{4} - \frac{\Delta_k}{5} > c(k) - \frac{\Delta_k}{2} = c(k^*) + \frac{\Delta_k}{2}.\]
Thus, we get $\mathrm{LCB}_{k^*}^{CB}(t) < c(k^*)+\Delta_k/8 < c(k^*)+\Delta_k/2 < \mathrm{LCB}_k^{CB}(t)$, so arm $k$ is not selected. Thus arm $k$ can only be pulled when $m_k(t) < n_0$, or $m_{k^*}(t) < n_{0,k^*}$, or $\mathcal{G}_t^c$ holds.  Hence \[\E[n_k(T)] \leq n_0 + n_{0,k^*} + \sum \limits_{t=1}^T\P(\mathcal{G}_t^c) \leq n_0 + n_{0,k^*} + \frac{8}{K}\cdot\frac{\pi^2}{6} = O(1).\] \hfill $\blacksquare$

\section{Proof of Theorem~\ref{thm:km_consistency}}\label{app:thm7}

Conditional on the chosen policy $k_t$ and the failure history up to $\tau_{Y_t^j}^{j,t}$, the residual lifetime $L_{Y_t^j + 1}^{j,t}$ has the same distribution $F$ as the original lifetime, and is independent of the deterministic quantity $C^{j,t} = k_t - \tau_{Y_t^j}^{j,t}$. This corresponds to Type I censoring with a fixed censoring time, which trivially satisfies the independent censoring assumption. Thus the uniform strong consistency follows from the Glivenko-Cantelli theorem for randomly right-censored data (see \citet{breslow_crowley_1974} and \citet{gill_1983}), and hence $\sup_{t \in [K]} |\widehat{S}_T(t) - S(t)| \xrightarrow{a.e.} 0$. Since $\widehat{F}_T(t) = 1 - \widehat{S}_T(t)$, uniform strong consistency of $\widehat{F}_T$ follows immediately. \hfill $\blacksquare$

\section{Proof of Theorem~\ref{thm:kmConvergence}}\label{app:thm8}

We first show that $\widehat{c}_T(k) \xrightarrow{a.s.} c(k)$ as $T \to \infty$ for every $k \in [K]$. For pointwise convergence of $\widehat{M}_T$, we proceed by induction on $k$.  The base case $\widehat{M}_T(0) = 0 = M(0)$ holds by definition.  Suppose $\widehat{M}_T(j) \xrightarrow{a.s.} M(j)$ for all $j < k$. The estimated renewal function satisfies \[\widehat{M}_T(k) = \widehat{F}_T(k) + \sum \limits_{j=1}^{k} \widehat{f}_T(j)\,\widehat{M}_T(k - j).\] By Theorem~\ref{thm:km_consistency}, $\widehat{F}_T(k) \xrightarrow{a.s.} F(k)$, so \[\widehat{f}_T(j) = \widehat{F}_T(j) - \widehat{F}_T(j-1) \xrightarrow{a.s.} f(j).\] Combined with the inductive hypothesis $\widehat{M}_T(k - j) \xrightarrow{a.s.} M(k - j)$, each of the finitely many terms converges almost surely, giving $\widehat{M}_T(k) \xrightarrow{a.s.} M(k)$. Thus for each $k \in [K]$, \[\widehat{c}_T(k) = \frac{c_b + c_f N\widehat{M}_T(k)}{k} \xrightarrow{a.s.} \frac{c_b + c_f N M(k)}{k} = c(k).\] Since $k^*$ is assumed to be unique, $\delta := \min_{k \neq k^*} \Delta_k > 0$.  For all sufficiently large $n_T$, $|\widehat{c}_T(k) - c(k)| < \delta/2$ for every $k \in [K]$, whence $\widehat{c}_T(k) > \widehat{c}_T(k^*)$ for all $k \neq k^*$ and thus $\widehat{k}_T^* = k^*$.  The almost-sure cost convergence $|\widehat{c}_T(\widehat{k}_T^*) - c(k^*)| \to 0$ follows immediately. \hfill $\blacksquare$

% Appendix here
% Options are (1) APPENDIX (with or without general title) or
%             (2) APPENDICES (if it has more than one unrelated sections)
% Outcomment the appropriate case if necessary
%
% \begin{APPENDIX}{<Title of the Appendix>}
% \end{APPENDIX}
%
%   or
%
% \begin{APPENDICES}
% \section{<Title of Section A>}
% \section{<Title of Section B>}
% etc
% \end{APPENDICES}

% Acknowledgments here

% References here (outcomment the appropriate case)

% CASE 1: BiBTeX used to constantly update the references
%   (while the paper is being written).
%\bibliographystyle{informs2014} % outcomment this and next line in Case 1
%\bibliography{<your bib file(s)>} % if more than one, comma separated

%\bibliographystyle{informs2014} % outcomment this and next line in Case 1
%\bibliography{sample} % if more than one, comma separated

% CASE 2: BiBTeX used to generate mypaper.bbl (to be further fine tuned)
%\input{mypaper.bbl} % outcomment this line in Case 2

%If you don't use BiBTex, you can manually itemize references as shown below.

%\bibliographystyle{nonumber}

\section*{Appendix 2: Algorithm Pseudocodes}
\label{sec:algorithms}

This appendix collects the pseudocodes for all seven algorithms developed in the main text. Algorithms~\ref{alg:vi-te} and~\ref{alg:value_iteration_avg} are the value iteration procedures for the two MDP formulations of Section~\ref{sec:setup} and require knowledge of the lifetime distribution $F$; they serve as benchmarks rather than data-driven policies. Algorithms~\ref{alg:ind_hoeffding} - \ref{alg:corr_bernstein} are the bandit-based online learning algorithms of Section~\ref{sec:bandit}, which operate without knowledge of $F$ and differ in whether arms are treated as independent or correlated and whether Hoeffding or Bernstein concentration inequalities are used to construct confidence bounds. Algorithm~\ref{alg:survival_renewal} is the Kaplan--Meier renewal algorithm of Section~\ref{sec:Ren+KM}, which explicitly estimates $F$ from censored operational data. Table~\ref{tab:complexity} of the main text summarises the time and space complexity of all seven algorithms.

\subsection*{A2.1 \quad Value Iteration for the Time-Elapsed MDP}

Algorithm~\ref{alg:vi-te} implements value iteration for the time-elapsed MDP of Section~\ref{sec:time-elapsed-mdp}. The state space has cardinality $K+1$, independent of fleet size $N$, making this procedure computationally inexpensive. The input requires the lifetime PMF $f$, from which the renewal function $M(\cdot)$ and renewal mass function $m(\cdot)$ are computed via the discrete renewal equation~(1). Convergence is monitored via the span seminorm $\mathrm{sp}(d^{(n)}) := \max_k d^{(n)}(k) - \min_k d^{(n)}(k)$ of successive iterates, which contracts to zero under the unichain condition established in Theorem~\ref{thm:te-vi-convergence}. Upon termination, the midpoint estimate $\hat{l}^*$ is $\varepsilon/2$-accurate and the extracted policy $\hat{\pi}^*$ is $\varepsilon/2$-optimal (Theorem~\ref{thm:te-vi-convergence}). The output threshold $k_s = k^*$ established in Theorem~\ref{thm:time-mdp} directly grounds the bandit formulation of Section~\ref{sec:bandit}.

\begin{algorithm}[h]
\caption{Value Iteration for the Time-Elapsed MDP\label{alg:vi-te}}
\KwIn{PMF $f$, parameters $K, N, c_f, c_b$, tolerance $\varepsilon > 0$}
\KwOut{$\hat{l}^*$, $w^{(n)}$, $\hat{\pi}^*$}
Compute $M(k)$ for $k = 1,\ldots,K$ via the renewal equation~\eqref{eqn:renewal}\;
Set $m(k) \gets M(k) - M(k-1)$ for $k = 1,\ldots,K$\;
Initialize $w^{(0)}(k) \gets 0$ for all $k \in \{0,1,\ldots,K\}$\;
\Repeat{$\mathrm{sp}(d^{(n)}) < \varepsilon$}{
  \For{$k = 1, \ldots, K-1$}{
    $u(k) \gets \min\bigl(c_fNm(k) + w^{(n)}(k+1),\ c_b + w^{(n)}(0)\bigr)$\;
  }
  $u(K) \gets c_b + w^{(n)}(0)$\;
  $w^{(n+1)}(k) \gets u(k) - u(0)$ for all $k$ \tcp*{Normalize at reference state}
  $d^{(n)}(k) \gets w^{(n+1)}(k) - w^{(n)}(k)$ for all $k$\;
}
$\hat{l}^* \gets u(0) + \tfrac{1}{2}\bigl(\min_k d^{(n)}(k) + \max_k d^{(n)}(k)\bigr)$\;
$\hat{\pi}^*(k) \gets \arg\min_{a \in \{0,1\}} Q_a(k)$ for all $k$, where $Q_0(k) = c_fNm(k) + w^{(n)}(k+1)$, $Q_1(k) = c_b + w^{(n)}(0)$\;
\end{algorithm}

\subsection*{A2.2\quad Value Iteration for the Age-Vector MDP}

Algorithm~\ref{alg:value_iteration_avg} implements value iteration for the age-vector MDP of Section~\ref{sec:MDP}. Unlike Algorithm~\ref{alg:vi-te}, the state space $\mathcal{S} = \{0,1,\ldots,K\}^N$ has cardinality $(K+1)^N$, so the procedure is tractable only for small fleet sizes $N$; in the numerical experiments of Section~5 we use $N=2$, giving $|\mathcal{S}| = 169$ states. The per-iteration cost is $O(N 2^N (K+1)^N)$ because computing $Q_0(\mathbf{s})$ at each state requires summing over all $2^N$ failure subsets. Convergence is again monitored via the span seminorm, and the termination guarantees of Theorem~\ref{thm:value_iter_convergence} apply. The output optimal average cost $\hat{g}^*$ and policy $\hat{\pi}^*$ serve as the gold-standard benchmark $g^*_{\mathrm{age}}$ against which all data-driven algorithms are evaluated in Section~\ref{sec:MDPvsDD_Comp}. By the symmetry established in Theorem~\ref{thm:mdp_structure}, the optimal policy depends only on the multiset of machine ages, reducing effective storage from $(K+1)^N$ to $\binom{K+N}{N}$ threshold values.

\begin{algorithm}[h]
\DontPrintSemicolon
\KwIn{CDF $F$, costs $(c_b,c_f)$, $N$, $K$, tolerance $\varepsilon>0$}
\KwOut{Optimal average cost $\widehat{g}$, value function $w^{(k)}$, policy $\hat\pi^*$}
\BlankLine
$w^{(0)}\equiv 0; \qquad k\leftarrow 0; \qquad \mathbf{s}_{\mathrm{ref}}\leftarrow\mathbf{0}$\;
\Repeat{$\mathrm{sp}(d^{(k)})<\varepsilon$}{
  \ForEach{$\mathbf{s}\in\mathcal{S}$}{
    $Q_0(\mathbf{s})\leftarrow c_f\sum \limits_{j}h(s^j{+}1)+
      \sum \limits_{A\subseteq[N]}\Bigl(\prod_{j\in A}h(s^j{+}1)
      \prod_{j\notin A}\bar{h}(s^j{+}1)\Bigr)w^{(k)}(\mathbf{s}^{[A]})$\;
    $v^{(k+1)}(\mathbf{s})\leftarrow\min\{Q_0(\mathbf{s}),\,c_b\}$\;
  }
  \lForEach{$\mathbf{s}\in\mathcal{S}$}{%
    $w^{(k+1)}(\mathbf{s})\leftarrow v^{(k+1)}(\mathbf{s})-v^{(k+1)}(\mathbf{s}_{\mathrm{ref}})$}
  $d^{(k)}\leftarrow w^{(k+1)}-w^{(k)}$;\;
  $\mathrm{sp}(d^{(k)})\leftarrow\max_{\mathbf{s}}\,d^{(k)}(\mathbf{s})-\min_{\mathbf{s}}\,d^{(k)}(\mathbf{s})$;\;
  $k\leftarrow k+1$\;
}
$\widehat{g}\leftarrow v^{(k)}(\mathbf{s}_{\mathrm{ref}})+
  \tfrac{1}{2}\left(\min_\mathbf{s} d^{(k)}(\mathbf{s})+\max_\mathbf{s} d^{(k)}(\mathbf{s})\right)$\;
\lForEach{$\mathbf{s}\in\mathcal{S}$}{$\hat\pi^*(\mathbf{s})\leftarrow
  \mathbf{1}[c_b\leq Q_0(\mathbf{s})]$ \hfill \texttt{// recompute $Q_0$ using $w^{(k)}$}}
\Return{$\widehat{g}; \qquad w^{(k)}; \qquad \hat\pi^*$}\;
\caption{Value Iteration for Average-Cost MDP}
\label{alg:value_iteration_avg}
\end{algorithm}

\subsection*{A2.3\quad Independent Hoeffding LCB}

Algorithm~\ref{alg:ind_hoeffding} is the simplest of the four bandit algorithms. Each arm $k \in [K]$ is treated as an independent learning problem, and the lower confidence bound is constructed from Hoeffding's inequality with exploration bonus $b\sqrt{2\log(t)/n_k(t)}$, where $b = c_f N$ is the half-range of the per-unit-time cost. The initialization phase pulls each arm once in order of increasing $k$ before switching to the LCB rule; this ensures every arm has at least one sample before confidence bounds are evaluated. The algorithm achieves $O((K-1)\log T)$ cumulative regret (Theorem~\ref{thm:indep_Hoeff}), matching the Lai--Robbins lower bound up to constants.

\begin{algorithm}[h]
\DontPrintSemicolon
\KwIn{$K$, $T$, $c_b$, $c_f$, $N$}\KwOut{$k_1,\ldots,k_T$}
\lForEach{$k\in[K]$}{$n_k\leftarrow 0; \qquad S_k\leftarrow 0$}
\For{$t=1$ \KwTo $T$}{
  \lIf{$\exists\,k:n_k=0$}{$k_t\leftarrow\min\{k:n_k=0\}$}
  \lElse{$k_t\leftarrow\argmin\limits_{k\in[K]}\left\{S_k/n_k - b\sqrt{2\log(t)/n_k}\right\}$}
  Observe total failures $Y_t$ over interval $k_t$\;
  $n_{k_t}\leftarrow n_{k_t}+1; \qquad S_{k_t}\leftarrow S_{k_t}+(c_b+c_f Y_t)/k_t$\;
}
\caption{Independent Hoeffding LCB}
\label{alg:ind_hoeffding}
\end{algorithm}

\subsection*{A2.4\quad Independent Bernstein LCB}

Algorithm~\ref{alg:ind_bernstein} replaces the Hoeffding exploration bonus with the empirical Bernstein bound of \citet{Audibert2009}, which substitutes the worst-case range $b^2$ with the sample variance $\hat{\sigma}^2_k(t)$. When $n_k(t) = 1$ the Hoeffding bonus is used as a fallback since the sample variance is undefined; for $n_k(t) \geq 2$ the Bernstein LCB of equation~(\ref{eq:bern-lcb-indep}) applies. The algorithm achieves the same $O((K-1)\log T)$ regret rate as Algorithm~\ref{alg:ind_hoeffding} (Theorem~\ref{thm:indep_Bern}) but with tighter instance-dependent constants when arm reward variances differ.

\begin{algorithm}[h]
\DontPrintSemicolon
\KwIn{$K$, $T$, $c_b$, $c_f$, $N$}\KwOut{$k_1,\ldots,k_T$}
\lForEach{$k\in[K]$}{$n_k\leftarrow 0; \qquad S_k\leftarrow 0; \qquad Q_k\leftarrow 0$}
\For{$t=1$ \KwTo $T$}{
  \lIf{$\exists\,k:n_k=0$}{$k_t\leftarrow\min\{k:n_k=0\}$}
  \Else{
    \lForEach{$k\in[K]$}{%
      $\widehat{c}_k\leftarrow S_k/n_k$;\;
      $\hat\sigma_k^2\leftarrow(Q_k - S_k^2/n_k)/(n_k-1)$ if $n_k\geq 2$, else $b^2$}
    $k_t\leftarrow\argmin\limits_{k\in[K]}\,\mathrm{LCB}_k^{IB}(t)$
      \tcp*{Eq.~\eqref{eq:bern-lcb-indep}; Hoeffding bound used when $n_k=1$}
  }
  Observe $Y_t; \quad c_t\leftarrow(c_b+c_f Y_t)/k_t; \quad n_{k_t}\leftarrow n_{k_t}+1; \quad S_{k_t}\leftarrow S_{k_t}+c_t; \quad Q_{k_t}\leftarrow Q_{k_t}+c_t^2$\;
}
\caption{Independent Bernstein LCB}
\label{alg:ind_bernstein}
\end{algorithm}

\subsection*{A2.5\quad Correlated Hoeffding LCB}

Algorithm~\ref{alg:corr_hoeffding} exploits the nested observation structure of Proposition~\ref{prop:nested_obs}: pulling arm $k$ can be used to infer the costs of all shorter intervals $j \leq k$ can be reconstructed at no additional cost. Consequently, the effective sample count for arm $j$ is $m_j(t) = \sum_{l=j}^K n_l(t) \geq n_j(t)$, and the LCB is computed using $m_j(t)$ in place of $n_j(t)$. The algorithm achieves $O((K-k^*)\log T)$ regret (Theorem~\ref{thm:corr_hoeff}), a strict improvement over the $O((K-1)\log T)$ bound of Algorithm~\ref{alg:ind_hoeffding} whenever $k^* > 1$, because arms $k < k^*$ require only $O(1)$ direct pulls.

\begin{algorithm}[h]
\DontPrintSemicolon
\KwIn{$K$, $T$, $c_b$, $c_f$, $N$}\KwOut{$k_1,\ldots,k_T$}
\lForEach{$k\in[K]$}{$m_k\leftarrow 0 \qquad S_k\leftarrow 0$}
\For{$t=1$ \KwTo $T$}{
  \lIf{$\exists\,k:m_k=0$}{$k_t\leftarrow\max\{k:m_k=0\}$}
  \lElse{$k_t\leftarrow\argmin\limits_{k\in[K]}\left\{S_k/m_k - b\sqrt{2\log(t)/m_k}\right\}$}
  Observe cumulative failures: $Z_j\leftarrow\sum \limits_{i=1}^j Y_i, \ \text{for } j=1,\ldots,k_t$\;
  \lForEach{$j=1,\ldots,k_t$}{%
    $m_j\leftarrow m_j+1 \qquad S_j\leftarrow S_j+(c_b+c_f Z_j)/j$}
}
\caption{Correlated Hoeffding LCB}
\label{alg:corr_hoeffding}
\end{algorithm}

\subsection*{A2.6\quad Correlated Bernstein LCB}

Algorithm~\ref{alg:corr_bernstein} combines the nested observation structure of Algorithm~\ref{alg:corr_hoeffding} with the empirical Bernstein confidence bound of Algorithm~\ref{alg:ind_bernstein}. The sample variance $\hat{\sigma}^{(m)2}_k(t)$ is now computed from the $m_k(t)$ correlated samples rather than the $n_k(t)$ direct pulls. The algorithm achieves $O((K-k^*)\log T)$ regret (Theorem~\ref{thm:corr_bern}), with the same improvement over independent arms as Algorithm~\ref{alg:corr_hoeffding} but potentially tighter constants due to variance adaptation. 

\begin{algorithm}[h]
\DontPrintSemicolon
\KwIn{$K$, $T$, $c_b$, $c_f$, $N$}\KwOut{$k_1,\ldots,k_T$}
\lForEach{$k\in[K]$}{$m_k\leftarrow 0 \qquad S_k\leftarrow 0 \qquad Q_k\leftarrow 0$}
\For{$t=1$ \KwTo $T$}{
  \lIf{$\exists\,k:m_k=0$}{$k_t\leftarrow\max\{k:m_k=0\}$}
  \Else{
    \lForEach{$k\in[K]$}{%
      $\hat\sigma_k^2\leftarrow(Q_k-S_k^2/m_k)/(m_k-1)$ if $m_k\geq 2$, else $b^2$}
    $k_t\leftarrow\argmin\limits_{k\in[K]}\,\mathrm{LCB}_k^{CB}(t)$
      \tcp*{From Eq.~\eqref{eq:bern-lcb}}
  }
  Set $C^{(j)}\leftarrow(c_b+c_f\sum \limits_{i=1}^j Y_i)/j$ for $j=1,\ldots,k_t$\;
  \lForEach{$j=1,\ldots,k_t$}{%
    $m_j\leftarrow m_j+1 \qquad S_j\leftarrow S_j+C^{(j)} \qquad Q_j\leftarrow Q_j+(C^{(j)})^2$}
}
\caption{Correlated Bernstein LCB}
\label{alg:corr_bernstein}
\end{algorithm}

\subsection*{A2.7\quad Kaplan--Meier Renewal Algorithm}

Algorithm~\ref{alg:survival_renewal} takes a fundamentally different approach from the bandit algorithms: rather than directly estimating the per-unit-time cost $c(k)$ for each arm, it estimates the lifetime distribution $F$ nonparametrically from the accumulated censored observations and uses the estimated $\hat{F}$ to compute $\hat{c}(k)$ via the renewal equation~\ref{eqn:renewal}. At each block replacement, every machine that has not failed contributes a right-censored observation, and every complete lifetime between failures contributes an uncensored observation; the Kaplan--Meier estimator $\hat{S}(t)$ of Theorem~\ref{thm:km_consistency} correctly incorporates both types. The algorithm has two hyperparameters: the exploration rate $\varepsilon$ governing the decaying exploration schedule $\varepsilon_t = \varepsilon/\sqrt{t}$, and the refit frequency $R$, which controls how often the KM estimator and renewal function are recomputed from the accumulated dataset $\mathcal{D}_T$. A cold-start phase ($t \leq 5$) ensures that all $K$ arms are pulled at least once before exploitation begins, providing the KM estimator with observations before the first policy update. Despite the absence of a formal regret guarantee, the algorithm identifies $k^*$ rapidly after the cold-start phase and thereafter incurs near-zero incremental regret, as established by Theorem~\ref{thm:kmConvergence} and confirmed in Section~\ref{sec:BanditKM_Simul}.

\begin{algorithm}[h]
\DontPrintSemicolon
\KwIn{$(c_b,c_f)$, $N$, $T_{\max}$, refit frequency $R$, $\varepsilon\in(0,1)$, $K$}
\KwOut{$\widehat{k}^*$, $\{\widehat{c}(k)\}_{k=1}^K$}
$\mathcal{D}\leftarrow\emptyset; \qquad t\leftarrow 0; \qquad \widehat{k}^*\leftarrow\mathrm{Uniform}([K])$\;
\While{$t<T_{\max}$}{
  $t\leftarrow t+1$\;
  \lIf{$t\leq 5$ \textbf{or} $\mathrm{Bernoulli}(\varepsilon/\sqrt{t})=1$}{%
    $k_t\leftarrow\mathrm{Uniform}([K])$}
  \lElse{$k_t\leftarrow\widehat{k}^*$}
  \For{$j=1$ \KwTo $N$}{
    $\tau\leftarrow 0$\;
    \While{$\tau<k_t$}{
      sample $\ell\sim F$\;
      \lIf{$\tau+\ell\leq k_t$}{$\mathcal{D}\leftarrow\mathcal{D}\cup\{(\ell,1)\}$;\;
        $\tau\leftarrow\tau+\ell$}
      \lElse{$\mathcal{D}\leftarrow\mathcal{D}\cup\{(k_t-\tau,\,0)\}$;\;
        \textbf{break}}
    }
  }
  \If{$t\bmod R=1$ \textbf{or} $t=6$}{
    Compute $\widehat{F}$ via KM on $\mathcal{D}; \quad \widehat{f}(s)\leftarrow\widehat{S}(s-1)-\widehat{S}(s); \quad \widehat{M}(0)\leftarrow 0$\;
    \lForEach{$s=1,\ldots,K$}{%
      $\widehat{M}(s)\leftarrow\sum \limits_{r=1}^{s}\widehat{f}(r)[1+\widehat{M}(s-r)]$}
    \lForEach{$k=1,\ldots,K$}{%
      $\widehat{c}(k)\leftarrow(c_b+c_f N\widehat{M}(k))/k$}
    $\widehat{k}^*\leftarrow\argmin\limits_{k\in[K]}\widehat{c}(k)$\;
  }
}
\Return{$\widehat{k}^*; \qquad \{\widehat{c}(k)\}_{k=1}^K$}\;
\caption{Kaplan--Meier Renewal Algorithm}
\label{alg:survival_renewal}
\end{algorithm}

\end{document}